\documentclass{article}[11pt]
 
\usepackage[utf8]{inputenc}
\usepackage{amsthm,amsmath,amsfonts}
\usepackage{fullpage}
\usepackage{amsmath,amsfonts,amsthm,amssymb,mathrsfs,dsfont,bbm} 
\usepackage{mathtools}
\usepackage{bbm}
\usepackage{hyperref}
\usepackage{subcaption}
\usepackage{verbatim}

\newcommand{\Reff}{R_{\text{eff}}}

\newcommand{\bE}{\mathbb{E}}
\newcommand{\R}{\mathbb{R}}

\newcommand{\Cov}{\mathrm{Cov}}
\newcommand{\Var}{\mathrm{Var}}
\newcommand{\E}{\bE}      

\DeclareMathOperator{\diag}{diag}
\DeclareMathOperator{\Tr}{Tr}

\DeclareMathOperator{\supp}{supp}

\DeclareMathOperator{\Proj}{Proj}

\usepackage{xcolor}
\newcommand{\fnote}[1]{\textcolor{red}{Fred: #1}}

\newcommand{\rnote}[1]{\textcolor{magenta}{Raghu: #1}}

\newcommand{\ignore}[1]{}

\newtheorem{theorem}{Theorem}

\newtheorem{lemma}{Lemma}
\newtheorem{fact}{Fact}
\theoremstyle{definition}
\newtheorem{defn}{Definition}
\newtheorem{remark}{Remark}
\newtheorem{example}{Example}

\title{Learning Some Popular Gaussian Graphical Models \\ without Condition Number Bounds}
\author{Jonathan Kelner\thanks{Department of Mathematics, Massachusetts Institute of Technology. Email: {\tt kelner@mit.edu}. This work was partially supported by NSF Award CCF-1565235} \\ MIT \and Frederic Koehler\thanks{Department of Mathematics, Massachusetts Institute of Technology. Email: {\tt fkoehler@mit.edu}. This work was supported in part by Ankur Moitra's ONR Young Investigator Award.} \\ MIT \and Raghu Meka\thanks{Department of Computer Science, UCLA. Email: {\tt raghum@cs.ucla.edu}. This work was supported by NSF CAREER Award CCF-1553605.} \\ UCLA \and Ankur Moitra\thanks{Department of Mathematics, Massachusetts Institute of Technology. Email: {\tt moitra@mit.edu}. This work was supported in part by NSF CAREER Award CCF-1453261, NSF Large CCF-1565235, a David and Lucile Packard Fellowship and an ONR Young Investigator Award.} \\ MIT}
\date{}

\begin{document}

\maketitle

\begin{abstract}
\normalsize
Gaussian Graphical Models (GGMs) have wide-ranging applications in machine learning and the natural and social sciences. In most of the settings in which they are applied, the number of observed samples is much smaller than the dimension and they are assumed to be sparse. While there are a variety of algorithms (e.g. Graphical Lasso, CLIME) that provably recover the graph structure with a logarithmic number of samples, they assume various conditions that require the precision matrix to be in some sense well-conditioned. 

Here we give the first fixed polynomial-time algorithms for learning attractive GGMs and walk-summable GGMs with a logarithmic number of samples without any such assumptions. In particular, our algorithms can tolerate strong dependencies among the variables. Our result for structure recovery in walk-summable GGMs is derived from a more general result for efficient sparse linear regression in walk-summable models without any norm dependencies.
We complement our results with experiments showing that many existing algorithms fail even in some simple settings where there are long dependency chains. Our algorithms do not.

\end{abstract}

\thispagestyle{empty}

\newpage

\setcounter{page}{1}

\section{Introduction}
\subsection{Background}

A Gaussian Graphical Model (GGM) in $n$ dimensions is a probability distribution with density
\[ p(X = x) = \frac{1}{\sqrt{(2\pi)^n \det \Sigma}} \exp\left(-(x - \mu)^T \Sigma^{-1} (x - \mu)/2\right)\]
where $\mu$ is the mean and $\Sigma$ is the covariance matrix. In other words, it is just a multivariate Gaussian. 
The important point is that their conditional independence structure is encoded by $\Theta = \Sigma^{-1}$, which is called the precision matrix and which we explain next. We can associate a graph to $\Theta$ which connects two nodes $i,j$ when $\Theta_{ij} \ne 0$. Now each node $i$ only interacts directly with its neighbors in the sense that $X_i$ is conditionally independent of every other node in the graphical model given its neighbors $(X_j)_{i \sim j}$. This is known as the Markov property, and is what led Dempster \cite{dempster1972covariance} to initiate the study of learning GGMs in the 1970s. An important measure of complexity for a GGM is its sparsity $d$, which measures the largest number of non-zero off-diagonal entries in $\Theta$ in any row.


GGMs have wide-ranging applications in machine learning and the natural and social sciences where they are one of the most popular ways to model  statistical relationships between observed variables. For example, they are used to infer the structure of gene regulatory networks (see e.g. \cite{wille2004sparse,menendez2010gene,schafer2005learning,basso2005reverse}) and to learn functional brain connectivity networks  \cite{huang2010learning,varoquaux2010brain}. In most of the settings in which they are applied, the number of observed samples is much smaller than the dimension. This means it is only possible to learn the GGM in a meaningful sense under some sort of sparsity assumption.

From a theoretical standpoint, there is vast literature on learning sparse GGMs under various assumptions. Many approaches focus on {\em sparsistency} \--- where the goal is to learn the sparsity pattern of $\Theta$ assuming some sort of lower bound on the strength of non-zero interactions. This is a natural objective because once the sparsity pattern is known, estimating the entries of $\Theta$ is straightforward (e.g. one can use ordinary least squares).
A popular approach is the Graphical Lasso\footnote{We note that \cite{friedman2008sparse} did not introduce this objective (see discussion there), but rather an \emph{optimization procedure} used to maximize it, and Graphical Lasso technically refers to this specific optimization procedure.} \cite{friedman2008sparse} which solves the following convex program:
$$ \max_{\Theta \succ 0} \log \det (\Theta) - \langle \widehat{\Sigma}, \Theta \rangle - \lambda \|\Theta\|_1$$
where $\widehat{\Sigma}$ is the empirical covariance matrix and $\|\Theta\|_1$ is the $\ell_1$ norm of the matrix as a vector. Since we are interested in settings where the number of samples is much smaller than the dimension, $\widehat{\Sigma}$ is a somewhat crude approximation to the true covariance. However, it is a good estimate when restricted to sparse directions. 
It is known that if $\Theta$ satisfies various conditions, which typically include an assumption similar to or stronger than the restricted eigenvalue (RE) condition (a condition which, in particular, lower bounds the smallest eigenvalue of any $2k \times 2k$ principal submatrix) then Graphical Lasso and related $\ell_1$ methods succeed in recovering the graph structure (see e.g. \cite{meinshausen2006high,zhou2009adaptive}). 
For the Graphical Lasso itself, under some incoherence assumptions on the precision matrix (beyond RE), it has been shown \cite{ravikumar2011high} that the sparsity pattern of the precision matrix can be accurately recovered from $O((1/\alpha^2)d^2\log(n))$ samples where $\alpha$ is an incoherence parameter  (omitting the dependence on some additional terms, and assuming the non-zero entries are bounded away from 0 and the variances are $O(1)$).
Yet another popular approach is the CLIME estimator which solves the following linear program
$$\min_{\Theta} \|\Theta\|_1 \mbox{ s.t. } \| \widehat{\Sigma} \Theta - I\|_\infty \leq \lambda$$
The analysis of CLIME assumes a bound $M$ on the maximum $\ell_1$-norm of any row of the inverse covariance (given that the $X_i$'s are standardized to unit variance). This is also a type of condition number assumption, although of a different nature than RE. It succeeds at structure recovery when given
$$m \gtrsim  C M^4 \log n$$
samples, again assuming the $\Theta_{ij}$ are either 0 or bounded away from 0.

While these works show that sparse GGMs can be estimated when the number of samples is logarithmic in the dimension, there is an important caveat in their guarantees. They need to assume that $\Theta$ is in some sense well-conditioned. However in the high-dimensional setting, this is a strong assumption which is violated by simple and natural models (e.g. a graphical model on a path), where these bounds turn out to be polynomial in the dimension. Furthermore, it is a fragile assumption that behaves poorly even under benign operations like rescaling the variables.
In this paper, we study some popular models of GGMs and show how to learn them efficiently in the low-sample regime, even when they are ill-conditioned. We complement our results with examples that
break both previous algorithms and our own algorithms for learning general sparse GGMs. This leaves open the question of whether some sparse GGMs may be computationally hard to learn with so few samples.
Finally, we show experimentally that popular approaches, like the Graphical Lasso and CLIME, do in fact need a polynomial in $n$ number of samples even for simple cases like discrete Gaussian Free Fields (GFFs). This appears to be the case whenever the corresponding graphs have large effective resistances. 

Our work was motivated by a recent paper of Misra, Vuffray and Lokhov \cite{misra18} which studied the question of how many samples are needed {\em information-theoretically} to learn sparse GGMs in the ill-conditioned case. They required only the following natural non-degeneracy condition: that for every non-zero entry $\Theta_{i,j}$ we have
$$ \kappa \leq \frac{|\Theta_{i,j}|}{\sqrt{\Theta_{ii} \Theta_{jj}}}$$
Intuitively, this condition requires that any non-zero interaction between $X_i$ and $X_j$ must be non-negligibly large compared to the geometric mean of their conditional variances, when we condition on all the other variables. Crucially, this does \emph{not} imply any sort of condition number bound, because it allows for the random variables to be strongly correlated (see e.g. the simple example (5) in \cite{misra18}).
They showed that it is possible to estimate the graph structure with
$$m \ge C \frac{d}{\kappa^2} \log n$$
samples. Thus, being well-conditioned is in fact not a prerequisite for being learnable with a logarithmic number of samples. On the other hand, the result of \cite{wang2010information} gives an information-theoretic lower bound\footnote{A subtle point arises when interpreting this bound, because $d$ and $\kappa$ are closely related quantities (see e.g. Lemma~\ref{lem:d-bounded-by-kappa} below) since the matrix $\Theta$ must be PSD. In the lower bound constructions of \cite{wang2010information} they have $d = O(1/\kappa)$ and the term dominating their bound depends only on $\kappa$.}
of $\Omega((1/\kappa^2)\log n)$  on the sample complexity for structure recovery.
So the upper bound of \cite{misra18} is not far from the lower bound of \cite{wang2010information} (it is unknown which of the lower bound or upper bound is loose).  

However, their algorithm runs in time $n^{O(d)}$, making it difficult to run except for small instances (see Section \ref{sec:simulations}). This is because their algorithm is based on a reduction to a sequence of sparse linear regression problems that can all be ill-conditioned. It is believed that such problems exhibit wide gaps between what is possible information theoretically and what is possible efficiently. For instance, it is known that the general \emph{sparse linear regression} problem (under fixed design) is $\mathbf{NP}$-hard in the proper learning setting\footnote{Where the algorithm is required to output a $d$-sparse estimator.} (see \cite{natarajan1995sparse, zhang2014lower}). 
Misra et al. solve the sparse linear regression problems using exhaustive search over $d$-size neighborhoods (hence the $n^{O(d)}$ time). This leads to the main question we study: 

Can we get efficient and practical algorithms for learning GGMs (run-time $\ll n^{o(d)}$) in some natural, but still ill-conditioned, cases? 

\subsection{Our Results}

We show that for some popular and widely-used classes of GGMs, it is possible to achieve both logarithmic sample complexity (the truly high-dimensional setting) and computational efficiency, even when $\Theta$ is ill-conditioned.  

\subsubsection*{Attractive GGMs}

First we study the class of attractive GGMs, in which the off-diagonal entries of $\Theta$ are non-positive. In terms of the correlation structure, this means that the variables are positively associated. A well-studied special case is the discrete Gaussian Free Field (GFF) where $\Theta$ is a principal submatrix of a graph Laplacian (i.e. we set some non-empty set of reference variables to zero as their boundary condition). This is a natural model because the Laplacian encourages ``smoothness'' with respect to the graph structure \---- if we think of the samples as random functions on the graph, then by integration by parts we see the log-likelihood of drawing a function is proportional to the $L^2$ norm of its discrete gradient \cite{sheffield2007gaussian}.
The GFF has a number of applications in active and semi-supervised learning (see \cite{zhu2003combining,zhu2003semi,mgs2013}, and more generally in the literature on Gaussian processes in machine learning \cite{rasmussen2003gaussian}). GFFs also have  important connections to random walks (for example, through Dynkin's second isomorphism theorem --- see \cite{ding2011cover}), and in the lattice case its scaling limit is an important generalization of Brownian motion that plays a key role in statistical physics and random surface theory \cite{ friedli2017statistical,sheffield2007gaussian}.
In the  GFF setting, $\Theta$ will be ill-conditioned whenever some pair of vertices have large \emph{effective resistance} between them (e.g., paths, rectangular grids, etc.,). This for example happens whenever there are nested sparse cuts which when collapsed lead to a long path resulting in variables having large (polynomial in $n$) variance.

We show experimentally (in Section~\ref{sec:simulations}) that simple examples like the union of a long path and some small cliques do indeed foil the Graphical Lasso and other popular methods. The fundamental issue is that none of the theoretical guarantees for Graphical Lasso and similar algorithms make sense for a long path. Intuitively, this is because GFFs on a path exhibit long-range correlations that violate the assumptions used in current works. 
This analysis reveals a blind spot of the Graphical Lasso: It performs poorly in the presence of long dependency chains, which can easily lead to missing some important statistical relationships in applications. 

We show that for attractive GGMs the conditional variance of some variable $X_i$ when we condition on a set $X_S$ is a monotonically decreasing and supermodular function of $S$. This fact was previously observed in the GFF setting (independently in \cite{mgs2013,mahalanabis2012subset}). We give a new, short proof of this fact using just basic linear algebra. 
We remark that Bresler et al.\@ \cite{learning-rbm} also used supermodularity, but of the influence function, to learn ferromagnetic Ising models with latent variables such as ferromagnetic RBMs; also, the use of submodularity for subset selection in linear regression appeared in \cite{das2011submodular}.
The supermodularity result allows us to give a simple greedy algorithm (with pruning) for learning the graph structure in the attractive case. In the literature, this is called a \emph{forward-backward method} \cite{lauritzen1996graphical}.

\begin{theorem}[Informal version of Theorem~\ref{thm:greedy-ferromagnetic}]\label{thm:greedy-ferromagnetic-informal}
Fix a $\kappa$-nondegenerate attractive GGM.
The \textsc{GreedyPrune} algorithm runs in polynomial time and returns the true neighborhood of every node $i$ with high probability with $m \geq C(d/\kappa^2)\log(1/\kappa)\log(n)$ samples, where $C$ is a universal constant. 
\end{theorem}

\noindent 
Our algorithm matches the sample complexity of the previous best (inefficient) algorithm for this setting \cite{misra18}
and obtains the optimal dependence on $\kappa$ for fixed $d$ (up to $\log(1/\kappa)$ factor; see discussion after Theorem~\ref{thm:search-and-validate-informal}).
To achieve this efficient sample complexity, we carefully analyze the alignment between the true decrement of conditional variance in one step, $\Var(X_i | X_S) - \Var(X_i | X_{S \cup \{j\}})$ and the noisy empirical decrement $\widehat{\Var}(X_i | X_S) - \widehat{\Var}(X_i | X_{S \cup \{j\}})$. A hurdle is that we need to control the differences $\widehat{\Var}(X_i | X_S) - \widehat{\Var}(X_i | X_{S \cup \{j\}})$ without assuming too much accuracy on the estimates $\widehat{\Var}(X_i | X_S)$ themselves. We do so we use matrix concentration and tools for analyzing the OLS estimator related to classical regression tests \cite{keener2011theoretical}. To complete the analysis, we need a new structural result for attractive GGMs which bounds the conditional variance after the first step of greedy, so that only a bounded number of iterations of greedy are required to learn a superset of the neighborhood. We prove this by reducing to the setting of discrete GFFs and using an electrical argument based on effective resistances.  

Prior work on learning attractive GGMs has focused on the Maximum Likelihood Estimator (MLE). This was shown to exist and be unique using connections to total positivity in \cite{slawski2015estimation,lauritzen2017maximum}. But we are not aware of
any sample complexity guarantees in the context of structure learning. It also is likely broken by the same examples (see Section~\ref{sec:simulations}) as the graphical lasso (since the constrained MLE is the same as the Graphical Lasso with zero regularization and a non-negativity constraint). 

\paragraph{Information-theoretic bounds.}
The previous literature leaves open the question of the information-theoretically optimal sample complexity for learning attractive GGMs. We resolve this question by demonstrating that a simple estimator based on $\ell_0$-constrained least squares, which we refer to as \textsc{SearchAndValidate}, achieves sample complexity matching the information-theoretic lower bounds of \cite{wang2010information} up to constants:
\begin{theorem}[Informal version of Theorem~\ref{thm:search-and-validate}]\label{thm:search-and-validate-informal}
  In a $\kappa$-nondegenerate attractive GGM, as long as $m = \Omega((1/\kappa^2)\log(n))$, with high probability Algorithm~\textsc{SearchAndValidate} returns the true neighborhood of every node $i$. This algorithm runs in time $O(n^{d + 1})$.
\end{theorem}
Here the corresponding sample complexity lower bound of $\Omega((1/\kappa^2)\log(n))$ follows from \cite{wang2010information} by flipping the signs of the parameter $a$ in one of their constructions. Equivalently, this comes down to the number of samples needed to distinguish the empty graph from a graph with a single $\kappa$-nondegenerate edge in an unknown location. 
We note that this bound does not depend on $d$, which may appear surprising. But it is actually not so strange, because $\kappa$-nondegeneracy implies an upper bound of $d \le 1/\kappa^2$ in $\kappa$-nondegenerate attractive GGMs --- see Lemma~\ref{lem:d-bounded-by-kappa}. We also give a version of the above result for general models with sample complexity $O(d\log(n)/\kappa^2)$ and run time $O(n^{d+1})$, giving a faster alternative to \cite{misra18} with the same sample complexity guarantee.

Theorem~\ref{thm:search-and-validate-informal} is proved by a careful analysis of the signal-vs-entropy tradeoff (in $\ell_0$-constrained regression) between choosing the correct support (which is best in expectation) and an incorrect support with $k$ disagreements for each $k$. As $k$ grows, the difference become worse in expectation, but there are roughly $n^k$ many sets which enables fitting the noise more effectively.  Precisely analyzing the differences in empirical risk again builds upon some classical ideas in regression testing \cite{keener2011theoretical}. We note that this result is also identifies an important barrier to improving the information theoretic lower bound of \cite{wang2010information} using similar lower bound instances. If this bound is not tight for general GGMs, it seems significantly new ideas will be needed to separate the sample complexity of learning attractive and non-attractive GGMs, as they must rely upon the ability of negative correlations to create nontrivial cancellations.

\subsubsection*{Walk-Summable GGMs}

While attractive GGMs are natural in some contexts, in others they are not. For example, in Genome Wide Association Schemes (GWASs), genes typically have inhibitory effects too. This leads us to another popular and well-studied class of GGMs, which includes as a special case all attractive GGMs: the \emph{walk-summable} models. These were introduced by Maliutov, Johnson and Willsky \cite{malioutov2006walk} to study the convergence properties of Gaussian Belief Propagation, generalizing previous work of Weiss and Fredman \cite{weiss2000correctness} for GGMs with SDD precision matrices. 
Walk-summable models also subsume other important classes of GGMs like {\em pairwise normalizable} and {\em non-frustrated} models \cite{malioutov2006walk}. A number of equivalent definitions are known for walk-summability --- perhaps the easiest to work with is that making all off-diagonal entries of $\Theta$ negative preserves the fact that $\Theta$ is positive definite. Perhaps less well known, it was shown in \cite{ruozzi2009graph} that walk-summable models are exactly those GGMs with SDD precision matrices under a rescaling of the coordinates.

The analysis of learning walk-summable models is considerably different from the attractive case, because supermodularity (and even weak submodularity \cite{das2011submodular}) of the conditional variance fail to hold -- see Section~\ref{sec:sdd-examples}. Regardless, we are still able to prove that \textsc{GreedyAndPrune} can learn all walk-summable models with sample complexity that scales logarithmically with $n$. To show this, we first reduce SDD models to generalized Laplacians. We then directly show that the greedy method makes significant progress in each step using further electrical arguments. Our analysis surprisingly shows that after a single step of greedy, the unknown sparse regression vector has small $\ell_1$-norm (independent of $n$ and scaling correctly with the noise level). 
This applies even for ill-conditioned models. The $\ell_1$-norm bound not only implies that greedy works, but also that appropriate invovations of $\ell_1$-based methods (like the Lasso) can now obtain good guarantees. We emphasize that such bounds do not hold without the first step of greedy. 


Concretely, we propose an algorithm called \textsc{HybridMB} 
based on this idea and show that it learns walk-summable GGMs without any condition number dependence.
The analysis of \textsc{HybridMB} uses the aforementioned structural results for walk-summable models and a statistical analysis for the regression problem arising after the greedy step. The regression analysis is similar in spirit to the usual generalization bounds for $\ell_1$-constrained regression (with care taken to handle the unbounded regressors and noise) but is more subtle. In particular, we must carefully take into account the interaction between fitting the unbounded coefficient on the greedily-selected variable (which by itself would be OLS) and the $\ell_1$-bounded coefficients on the other variables (which by itself would be the Lasso) from data. 
The key insight here is that if we fit all of these coefficients together at once, then this interaction can be eliminated under an unknown (to the algorithm) reparameterization of our function class of predictors; since empirical risk minimization does not depend on the parameterization of our class, we are then able to apply the fixed point machinery of \cite{mendelson2014learning} and obtain sharp bounds on the generalization error.

\begin{theorem}[Informal version of Theorem~\ref{thm:structure-learning-via-hybrid}]\label{thm:sdd-hybrid-informal}
Fix a walk-summable, $\kappa$-nondegenerate GGM.
Algorithm \textsc{HybridMB} runs in polynomial time and returns the true neighborhood of every node $i$ with high probability given $m \geq C(d/\kappa^4)\log(n)$ samples, where $C$ is a universal constant. 
\end{theorem}
Prior to our work, Anandkumar, Tan, Huang and Willsky \cite{anandkumar2012high} gave an $n^{O(d)}$ time algorithm for learning walk-summable models which also required some additional assumptions. 
As mentioned, we also give a similar result to the above for the Algorithm~\textsc{GreedyAndPrune}, albeit with slightly worse dependence on $d$ and $\kappa$ --- see Theorem~\ref{thm:greedy-and-prune-sdd}.

The above structure learning result requires $\kappa$-nondegeneracy and sparsity of the entire model. But it is proved using the following general result for sparse linear regression, which requires only a walk-summability assumption:
\begin{theorem}[Informal version of Theorem~\ref{thm:ws-regression}]
  Suppose that $Y = w \cdot X + \xi$ where $w$ is $d$-sparse, $\xi \sim N(0,\sigma^2)$ is independent of multivariate Gaussian r.v. $X \sim N(0,\Sigma)$, and suppose that the joint distribution of $(X_1,\ldots,X_n,Y)$ is walk-summable. Given $m$ samples from this model, \textsc{WS-Regression} runs in polynomial time and returns $\hat{w}$ such that
  \[ \E[(w \cdot X - \hat{w} \cdot X)^2] = O(\sigma^2\sqrt{d\log(n)/m}) \]
  with high probability.
\end{theorem}

Although this result gives a ``slow rate'' of $\sqrt{1/m}$, it is quite different from the usual ``slow rate'' for the Lasso. The latter typically has error of the form $O(\sigma R W \sqrt{\log(d)/m})$ where $R$ is an $\ell_1$ norm bound on $w$ and $W$ is an $\ell_{\infty}$ bound on $X$, see e.g. \cite{rigollet2015high}. Concretely, when $R,W = \Theta(1)$ our Theorem~\ref{thm:ws-regression} can achieve error on the order of the noise level $\sigma^2$ using $O(d \log(n))$ samples whereas standard slow rate results only achieve error on the order of $\sigma$. This difference is crucial for achieving structure recovery from $O(\log n)$ samples as $\sigma$ can be very small (shrinking as $n$ grows) in our applications.
Compared to $\ell_0$-constrained least squares, which requires runtime $O(n^d)$, the above result is computationally efficient and has the correct dependence on $d,\sigma^2$.


\paragraph{General Models.} We note that our methods (\textsc{GreedyPrune}, \textsc{HybridMB}) also essentially recover the sample complexity bounds of \cite{cai2016estimating} under their assumptions (that the entries of the inverse precision matrix are bounded) --- see Remark~\ref{rmk:l1-bounded}. This result does not require the greedy step.

\subsection{Further Discussion}
There are interesting parallels and also significant differences between the situation for learning GGMs and Ising models. For Ising models, Bresler \cite{bresler2015efficiently} gave a simple greedy algorithm that builds a superset of the neighborhood around each node and then prunes to learn the true graph structure. For an $n$ node Ising model with degree $d$ and upper and lower bounds on the interaction strength of any nonzero edge and upper bounds on the external field, the algorithm runs in $f(d) \mbox{poly}(n)$ time and uses $f(d) \log n$ samples. In particular, this greedy algorithm is able to perform structure learning in Ising models even when they exhibit long range correlations, which previous results could not handle. However in our setting, and unlike the previously described situation for Ising models, variables have real values and can have arbitrarily small or large variance. It turns out this changes the problem dramatically, as it means that the inter-node fluctuations in the random field (which contribute to the variance of the field $X_i$ at node $i$) may be orders of magnitude larger than the per-node fluctuations (corresponding to the conditional variance of $X_i$). As a result of this difference, greedy methods actually fail to learn general GGMs from $O(\text{polylog}(n))$ samples (see Appendix~\ref{apdx:hard-examples}); therefore, any analysis of greedy methods must rely on structural results of a subclass of models. The same issue also prevents us from learning the model directly from $\ell_1$-constrained regression results as in \cite{vuffray2016interaction,klivans2017learning} --- in fact, we will see in Section~\ref{sec:simulations} that natural methods based only on $\ell_1$ regularization fail even in some relatively simple attractive GGMs (where greedy works).

As previously mentioned, Das and Kempe \cite{das2011submodular} studied the problem of sparse regression without assuming the restricted eigenvalue condition. While in sparse regression, in order to learn the parameters accurately (in additive error) some bound on the condition number is needed, they studied the problem of selecting a subset of columns that maximizes squared multiple correlation (a.k.a. minimizes mean squared error). 
They then gave approximation guarantees for many popular algorithms, including greedy, under an approximate submodularity condition and assuming access to the true joint covariance matrix (in other words, they studied this as a purely algorithmic problem while disregarding the effect of noise). Our algorithm for attractive models follows the same supermodularity-based strategy, but does not assume access to the true covariance matrix --- instead, we carefully analyze its statistical performance by studying the interaction between the greedy iteration and noise. In the general setting of walk-summable GGMs, we show the conditional variance does not satisfy an approximate supermodularity condition with any constant submodularity ratio. (See discussion in Remark~\ref{rmk:submodularity-ratio}.) 

In the literature on sparse regression, it is well known that the analyses of the Lasso which work well in a compresssed sensing style setting (i.e. with restricted eigenvalues, incoherent columns, etc.) may not be the correct tool to use when the coordinates of $X$ (columns of the design matrix) are highly correlated --- see e.g. \cite{van2013lasso,hebiri2012correlations,candes2014towards}. For example, the work of Koltchinskii and Minsker \cite{koltchinskii2014l_1} discusses this issue in the context of Brownian motion and other situations and develops general new guarantees for $\ell_1$-penalized regression which apply under correlated design (as well as infinite dimensional settings). Their result, for example, gives improved bounds when the response variable $Y$ is a linear combination of measurements of a simple random walk $X_{t_1},X_{t_2},\ldots,X_{t_k}$ when $t_1,\ldots,t_k$ are well-separated. We note that their setup and result is incomparable to ours, as the joint distribution of $(X,Y)$ in their model need not be walk-summable, and for structure recovery of random-walk like GGMs it is crucial to analyze the case where $t_1,t_2,\ldots$ are not well-separated. It would be interesting to see if the ideas used in Algorithms \textsc{HybridMB} and \textsc{GreedyAndPrune} are applicable to other regression setups with correlated design.

\section{Preliminaries}\label{sec:preliminaries}

In this section we set out some notation and basic facts about GGMs which will be used throughout.

\paragraph{Notations.} Given a GGM with precision matrix $\Theta$, $d$ will always denote the maximum degree of the underlying graph. Thus, $\Theta$ has at most $d + 1$ nonzero entries in each row. For a vector $x$ and index $i$, $X_{\sim i} = ((X_j): j \neq i)$. For a square matix $S \in \R^{k \times k}$ and $I \subseteq [k]$, $S_I$ denotes the $I \times I$ principal submatrix of $S$. 

We recall that conditioning on $X_i = x_i$ for any $x_i$ yields a new GGM with the precision matrix having row $i$ and column $i$ deleted. In particular, the conditional precision matrix does not depend on the value of $x_i$ chosen. Similarly, the value of the mean $\mu$ does not affect the covariance structure at all --- so $\mu$ does not play an interesting role in the structure learning problem and the reader may safely assume $\mu = 0$. We summarize the facts that we use the most below. 

\begin{fact}[\cite{lauritzen2011elements}]\label{fact:ggmfacts}
Let $X$ be drawn from a mean $0$ GGM with precision matrix $\Theta$. Then, for any $i$, $X_i | X_{\sim i} = x_{\sim i}$ is distributed as $N(\langle w^{(i)}, x_{\sim i} \rangle, 1/\Theta_{ii})$ where $w^{(i)}$ is the vector with $w^{(i)}_j = - \Theta_{ij}/\Theta_{ii}$. 
\end{fact}

\ignore{
The fact follows from expanding the density of 
First we recall that the density of $X_1$ given $X_{\sim 1} = x_{\sim 1}$ is given by
\begin{align}
    p(X_1 = x_1 | X_{\sim 1} = x_{\sim 1}) 
    &\propto \exp(-\Theta_{11} (x_{1} - \mu_1)^2/2 - x_1 \Theta_{1,\sim 1}^T (x_{\sim 1} - \mu_{\sim 1})) \\
    &\propto \exp\left(-\Theta_{11}\left(x_{1} - \mu_1 + \frac{\Theta^T_{1,\sim 1}}{\Theta_{11}} (x_{\sim 1} - \mu_{\sim 1})\right)^2/2\right)\label{eqn:ggm-conditional}
\end{align}
which is a Gaussian with mean $\mu_1 - \frac{\Theta^T_{1,\sim 1}}{\Theta_{11}} (x_{\sim 1} - \mu_{\sim 1})$ and variance $1/\Theta_{11}$.}

Thus, if we fix an index $i$, then samples $X$ from the GGM can be interpreted as a linear regression problem as $(X_{\sim i}, X_i)$ where $X_i = \langle w^{(i)}, X_i \rangle + N(0, 1/\Theta_{ii})$. 
This establishes the basic connection between learning GGMs and linear regression: if we can solve the above regression problem well, perhaps we can recover the non-zero entries of $\Theta$ from the coefficients. But as is well known in the literature, just fitting the coefficients using ordinary least squares is not sufficient (or necessarily possible) as we have very few samples. 

By positive definiteness, we have $\Theta_{i,i} \ge 0$ and $\Theta_{i,i} \Theta_{j,j} - \Theta_{i,j}^2 \ge 0$, or equivalently
$0 \le \frac{|\Theta_{i,j}|}{\sqrt{\Theta_{i,i} \Theta_{j,j}}} \le 1$.
To identify the graph we need the present edges to not be too weak. So it makes sense to assume (following the notation of \cite{anandkumar2012high,misra18}) there is a $\kappa > 0$ such that
\begin{equation}\label{eqn:kappa-assumption}
\kappa \le \frac{|\Theta_{i,j}|}{\sqrt{\Theta_{i,i} \Theta_{j,j}}} \le 1
\end{equation}
\begin{defn}[\cite{anandkumar2012high,misra18}]
We say a GGM is $\kappa$-nondegenerate if it satisfies \eqref{eqn:kappa-assumption} for all $i,j$ such that $\Theta_{ij} \ne 0$.
\end{defn}

\paragraph{Conditional Variance.} Conditional variances of the form $\Var(X_i| X_S)$ play a central role in all our algorithms. We first review the basic definition and some of their properties. 
\begin{defn}[Conditional Variance]
For $X$ an arbitrary real-valued random variable and $Y$ an arbitrary random variable or collection of random variables on the same probability space, let\footnote{In an alternate convention which we do not use, $\Var(X|Y)$ is defined to be the random variable $\E[(X - \E[X | Y])^2 | Y]$ and our definition is the same as $\E \Var(X|Y)$.}
\[ \Var(X|Y) := \E[(X - \E[X | Y])^2]. \]
\end{defn}
By the Pythagorean Theorem, conditional variance obeys the \emph{law of total variance} \cite{blitzstein2014introduction}:
\[ \Var(X) = \Var(X|Y) + \Var(\E[X | Y]). \]
and more generally, $\Var(X | Y) = \Var(X | Y,Z) + \Var(\E[X | Y,Z] | Y)$.
The last identity is also sometimes referred to as the law of total conditional variance.


The $\kappa$-nondegeneracy assumption implies a quantitative lower bound on conditional variances $\Var(X_i | X_S)$ when the conditioning set does not include all of $i$'s neighbors.
\begin{lemma}\label{lem:kappa-variance}
Fix a node $i$ in a $\kappa$-nondegenerate GGM, and let $S$ be set
of nodes not containing all neighbors of $i$. Then
\[ \Var(X_i | X_S) \ge \frac{1 + \kappa^2}{\Theta_{ii}} \]
\end{lemma}
\begin{proof}
Let $j \notin S$ be a neighbor of $i$. By the law of total conditional variance, we have
$$\Var(X_i | X_S) = \Var(X_i | X_{\sim i}) + \Var(\E[X_i | X_{\sim i}] | X_S) = \frac{1}{\Theta_{ii}} + \Var(\E[X_i | X_{\sim i}] | X_S),$$
where in the last equality we used Fact \ref{fact:ggmfacts}. Thus, as $\E[f^2] \ge \Var(f)$, and the definition of $\kappa$-nondegeneracy
\begin{align*}
 \Var(X_i | X_S)  - \frac{1}{\Theta_{ii}} = \Var(\E[X_i | X_{\sim i}] | X_S)
&= \E[(\E[X_i | X_{\sim i}] - \E[X_i | X_S])^2] \\
&\ge \Var(\E[X_i | X_{\sim i}] - \E[X_i | X_S] | X_{\sim j}) = \frac{\Theta_{ij}^2}{\Theta_{ii}^2 \Theta_{jj}} \ge \frac{\kappa^2}{\Theta_{ii}}
\end{align*}
where the last equality follows from Fact~\ref{fact:ggmfacts} and the last inequality
is by the definition of $\kappa$.
The Lemma follows by rearranging. \end{proof}

The following basic fact about Gaussians will be useful:
\begin{lemma}\label{lem:gaussian-conditional-expectation}
  If $X$ and $Y$ are jointly Gaussian random variables then 
  $\E[X | Y] = \E[X] + \frac{\Cov(X,Y)}{\Var(Y)}(Y - \E[Y])$
  and
  $\Var(X) - \Var(X | Y) = \frac{\Cov(X,Y)^2}{\Var(Y)}$.
\end{lemma}
\begin{proof}
Because the random variables are jointly Gaussian, we know that $\E[X | Y]$ must be an affine function of $Y$. From $\E[\E[X | Y]] = \E[X]$ and $\Cov(\E[X | Y], Y) = \Cov(X,Y)$ the coefficients are determined, proving the first formula. Then the second formula follows from the law of total variance, $\Var(X) - \Var(X | Y) = \Var(\E[X | Y])$.
\end{proof}

We will also use the following concentration inequality often. Recall that a $\chi^2$-random variable with $D$ degrees of freedom is just $\sum_{i = 1}^D Z_i^2$ where $Z_i \sim N(0,1)$ are independent standard Gaussians.
\begin{lemma}[Lemma 1, \cite{laurent2000adaptive}]\label{lem:chi-squared-concentration}
  Suppose $U$ is $\chi^2$-distributed with $D$ degrees of freedom. Then
  $\Pr(U - D \ge 2\sqrt{D\log(1/\delta)} + 2\log(1/\delta)) \le \delta$
  and
  $\Pr(D - U \ge 2\sqrt{D\log(1/\delta)}) \le \delta$.
  In particular, $U \le 2D$ with probability at least $1 - \delta$ as long as $D \ge 8\log(1/\delta)$.
\end{lemma}

\section{Structural results for walk-summable models}
\subsection{Background: Walk-Summable Models are SDD after rescaling}
\begin{defn}[\cite{malioutov2006walk}]
A Gaussian Graphical Model with invertible precision matrix $\Theta \succ 0$ is \emph{walk-summable} if $D - \overline{A} \succ 0$ where $\Theta = D - A$ decomposes $\Theta$ into diagonal and off-diagonal components, and $\overline{A}$ is the matrix with $\overline{A}_{ij} = |A_{ij}|$.
\end{defn}
It is well-known (and immediate) that the class of walk-summable matrices includes the class of SDD matrices. Indeed, the motivation for introducing walk-summable matrices was to generalize the notion of SDD matrices.
\begin{defn}
A matrix $M$ is \emph{symmetric diagonally dominant} (SDD) if it is symmetric and $M_{ii} \ge \sum_{j: j \ne i} |M_{ij}|$ for every $i$. 
\end{defn}
Perhaps less well-known, it was observed in \cite{ruozzi2009graph} that a natural converse holds: all walk-summable matrices are simply rescaled SDD matrices, where the rescaling is in the natural sense for a bilinear form. 
Furthermore, this rescaling is easy to find algorithmically (if we have access to $\Theta$), requiring just a top eigenvector computation.
\begin{theorem}[Theorem 4.2 of \cite{ruozzi2009graph}]\label{thm:rescaling}
Suppose $\Theta$ is walk-summable. Then there exists a diagonal matrix $D$ with positive entries such that $D \Theta D$ is an SDD matrix.
\end{theorem}
\begin{proof}
We include the proof for completeness --- it is the same as in \cite{ruozzi2009graph}.

First, we observe that we can reduce to the case $\diag(\Theta) = \vec{1}$ by replacing $\Theta$ by $D_1 \Theta D_1$ where $D_1$ is the diagonal matrix with $(D_1)_{ii} = 1/\sqrt{\Theta_{ii}}$. Next, let $\overline{\Theta} = I - \overline{A}$ and note that when we write the decomposition $0 \prec \overline{\Theta} = I - \overline{A}$ that $\overline{A}$ has all nonnegative entries, so we can apply the Perron-Frobenius Theorem to find an eigenvector $v$ with positive entries and eigenvalue $\lambda = \|\overline{A}\| < 1$. Now define $D_2 = \diag(v)$, and we claim that $D_2 \Theta D_2$ is an SDD matrix. It suffices to check that $0 \le D_2 \overline{\Theta} D_2 \vec{1} = D_2 \overline{\Theta} v$ entry-wise, and because $D_2$ is diagonal with nonnegative entries it suffices to check that $\overline{\Theta} v \ge 0$. This follows as
\[ \overline{\Theta} v = (I - \overline{A})v = (1 - \lambda)v \ge 0 \]
entrywise.
\end{proof}

We note that while that we are not aware of the above statement (Theorem~\ref{thm:rescaling}) appearing before the work of \cite{ruozzi2009graph}, related statements about $Z$-matrices (matrices, not necessarily symmetric, which have only negative off-diagonal entries) and $M$-matrices have been known for a long time in the linear algebra literature --- see for example Theorem 4.3 of \cite{fiedler1962matrices}.

\begin{example}
In Example 1 of \cite{malioutov2006walk} it was observed that the matrix
\[
\begin{bmatrix}
1  & -r & r & r \\
-r & 1  & r & 0 \\
r  & r  & 1 & r \\
r  & 0  & r & 1
\end{bmatrix}
\]
itself stops being SDD when $r > 1/3$, but remains walk-summable until a little past $r = 0.39$. When $r = 0.39$, the corresponding Perron-Frobenius eigenvector for $\overline{A}$ is roughly $(0.557,0.435,0.557,0.435)$ and applying the rescaling from Theorem~\ref{thm:rescaling} we get
\[
\begin{bmatrix}
0.310634 & -0.0945889 & 0.121147 & 0.0945889 \\
-0.0945889 & 0.189366 & 0.0945889 & 0. \\ 0.121147 & 0.0945889 & 0.310634 & 0.0945889 \\ 
0.0945889 & 0. & 0.0945889 & 0.189366 \end{bmatrix}
\]
which is an SDD matrix.
\end{example}
The SDD rescaling given by Theorem~\ref{thm:rescaling} will play a key role in our analysis.
Conceptually, converting a walk-summable matrix to its SDD form is a way to take the extra degrees of freedom in the model specification (arbitraryness in the scaling of the $X_i$) and fix them in a way that is useful in the analysis -- one instance of a very common phenomenon in mathematics, referred to as ``gauge fixing'' in some contexts. In particular, under the SDD rescaling there are meaningful relations between the different rows of $\Theta$ which fail to hold in general.

\subsection{Background: SDD systems, Laplacians, and electrical flows}\label{sec:electrical-flows}
\begin{defn}
A matrix $L$ is a \emph{generalized Laplacian} if it is SDD and for every $i \ne j$, $L_{ij} \le 0$. We think of this graph theoretically as the Laplacian of the weighted graph with edge weights $-L_{ij}$ between distinct $i$ and $j$ and self loops of weight $L_{ii} - \sum_{j \ne i} |L_{ij}|$ at vertex $i$.
\end{defn}
We review the standard reduction between solving SDD systems and Laplacian systems. Suppose $\Theta$ is an SDD matrix. Then we can write
$\Theta = L - P$
where $L$ is a (generalized) Laplacian having positive entries on the diagonal and nonnegative entries off the diagonal, and $P$ has negative off-diagonal entries and corresponds to the positive off-diagonal entries of $\Theta$. Now we observe that
\begin{equation}\label{eqn:lift-laplacian}
\begin{bmatrix} L & P \\ P & L \end{bmatrix} \begin{bmatrix} x \\ -x \end{bmatrix} = \begin{bmatrix} \Theta x \\ - \Theta x \end{bmatrix}
\end{equation}
and the left matrix is itself a (generalized) Laplacian matrix on a weighted graph which we will refer to as the ``lifted graph''. 

The inverse of a Laplacian has a natural interpretation in terms of electrical flows, where the edge weights are interpreted as conductances of resistors. In the next Lemma we summarize the relevant facts about this interpretation, as can be found in e.g. \cite{bollobas2013modern}
\begin{lemma}
\label{lem:Reff-facts}
Suppose that $L$ is a (generalized) Laplacian matrix. Then if $L^+$ is the pseudo-inverse of $L$, and we define the \emph{effective resistance} $\Reff(i,j) := (e_i - e_j)^TL^+(e_i - e_j)$ then $\Reff$ satisfies:
\begin{itemize}
    \item (Nonnegativity) $\Reff(i,j) \ge 0$.
    \item (Monotonicity) $\Reff(i,j) \le \frac{1}{|L_{ij}|}$, and more generally $\Reff$ decreases when adding edges to the original adjacency matrix.
    \item (Triangle inequality)
$\Reff(i,k) \le \Reff(i,j) + \Reff(j,k)$
for any $i,j,k$.
\end{itemize}
\end{lemma}
\subsection{Key structural results for Walk-Summable GGM}
First we prove a fundamental fact about $\kappa$-nondegeneracy in walk-summable models, mentioned earlier: the maximum degree $d$ always satisfies $d = O(1/\kappa^2)$ in $\kappa$-nondegenerate walk-summable models. This result is tight for star graphs.
\begin{lemma}\label{lem:d-bounded-by-kappa}
In a $\kappa$-nondegenerate walk-summable GGM, the maximum degree of any node is at most $1/\kappa^2$.
\end{lemma}
\begin{proof}
Rescale the coordinates so that the diagonal of $\Theta$ is all-1s, and reorder them so that $X_1$ corresponds to the node of maximum degree $d$ with neighbors $2,\ldots,d+1$. Define $\overline{\Theta}$ to be the sign-flipped version of $\Theta$ such that all off-diagonal entries are negative; by the definition of walk-summability we know $\overline{\Theta}$ is still PSD.
Let $v = (1,\kappa,\ldots,\kappa) \in \mathbb{R}^{d + 1}$ and $S = \{1,\ldots, d + 1\}$; then using that the off-diagonals are negative, $\kappa$-nondegeneracy we find that $\Theta_{d + 1, d + 1} v \le (1 - d\kappa^2, 0,\ldots,0)$ coordinate-wise, hence using $\overline{\Theta} \succeq 0$ we find
\[ 0 \le v^T \Theta_{d + 1, d + 1} v \le v^T (1 - d\kappa^2, 0, \ldots, 0) = 1 - d\kappa^2. \]
Rearranging we see that $d \le 1/\kappa^2$.
\end{proof}

In the remainder of this subsection we prove some key structural results about walk-summable/SDD GGM using the SDD to Laplacian reduction and the electrical interpretation of the inverse Laplacian; these results will be crucial for analyzing the algorithms for both attractive and general walk-summable GGMs.

The following key Lemma, which shows that the variance between two adjacent random variables in the SDD GFF cannot differ by too much, will be crucial in the analysis of our algorithm in non-attractive models. Why is this useful? Informally, this is because for the greedy method to significantly reduce the variance of node $i$, at least one neighbor of $i$ needs to provide a good ``signal-to-noise ratio'' for estimating $X_i$, and under the SDD scaling, this inequality shows that the neighbors do not have too much extra noise (compared to $|\Theta_{ij}|$ which roughly corresponds to the level of signal between nodes $i$ and $j$).
\begin{lemma}\label{lem:sdd-smooth-variance}
Suppose that $\Theta$ is an invertible SDD matrix. Let $\Sigma = \Theta^{-1}$.
If $\Theta_{ij} \ne 0$, then
\[ \Sigma_{ii} \le 1/|\Theta_{ij}| + \Sigma_{jj}. \]
\end{lemma}
\begin{proof}
Let $M$ be the generalized Laplacian matrix resulting from applying the SDD to Laplacian reduction from $\Sigma$, i.e. $M$ is the left hand-side of \eqref{eqn:lift-laplacian}. Let the standard basis for $\mathbb{R}^{2n}$ be denoted $e_1,\ldots,e_n,e'_1,\ldots,e'_n$.
Observe from \eqref{eqn:lift-laplacian} that
\[ \Sigma_{ii} = e_i^T \Theta^{-1} e_i = e_i^T M^+ (e_i - e'_i) = \frac{1}{2} (e_i - e'_i)^T M^+ (e_i - e'_i). \]

Let node label $i$ be the node corresponding to $e_i$ in the graph corresponding to $M$, and label $i'$ be that corresponding to $e'_i$.
Observe that in the graph corresponding to $M$, either $i$ is adjacent to $j$ and $i'$ is adjacent to $j'$, or $i$ is adjacent to $j'$ and $i'$ is adjacent to $j$. Let $r = \Reff(i,j)$ in the first case and $r = \Reff(i,j')$ in the second case.
By the triangle inequality (Lemma~\ref{lem:Reff-facts}) and monotonicity of effective resistance (Lemma~\ref{lem:Reff-facts}),
\[ 2\Sigma_{ii} = \Reff(i,i') \le 2r + \Reff(j,j') \le 2/|\Theta_{ij}| + 2\Sigma_{jj} \]
which proves the result.
\end{proof}
\begin{remark}
Note that the above Lemma is for $\Theta$ under the true SDD scaling. It would not make sense for general $\Theta$, because the left hand and right hand sides do not scale in the same way. 
\end{remark}
The following two lemmas show that in a SDD GGM, the variance of a single node can be bounded as long as we condition on any of its neighbors. In comparison, if we don't condition on anything then the variance can be arbitrarily large: consider the Laplacian of any graph plus a small multiple of the identity.
\begin{lemma}\label{lem:bound-after-conditioning-ij}
Suppose that $i$ is a non-isolated node in an SDD GGM. Then for any neighbor $j$ it holds that
\[ \Var(X_i | X_j) \le \frac{1}{|\Theta_{ij}|} \]
\end{lemma}
\begin{proof}
This result can be obtained from the previous Lemma~\ref{lem:sdd-smooth-variance} by taking an appropriate limit which sends $\Sigma_{jj} \to 0$. We give an alternate and direct proof below.

Apply the SDD to Laplacian reduction to the precision matrix (with row and column $j$ eliminated) as in Lemma~\ref{lem:sdd-smooth-variance} to get a generalized Laplacian $L$, and then form the standard Laplacian $M$ 
by adding an additional row and column $n + 1$ with $M_{i,n + 1} = L_{ii} - \sum_{j = 1}^n L_{ij}$ and $M_{n+1,n+1} = \sum_{j = 1}^n M_{j,n}$. Then $u = Lv$ iff there exists $z$ s.t. $(u,z) = M(v,0)$ where $(v,0)$ denotes the vector in $\mathbb{R}^{n + 1}$ given by adding final coordinate 0. Furthermore it must be that $\sum_i u_i + z = 0$ because $(u,z)$ lies in the span of $M$. Using the relation between $L$ and $M$ and the triangle inequality and monotonicity (Lemma~\ref{lem:Reff-facts}) through the added node $n + 1$ we observe
\begin{align*} 
\Var(X_i | X_j) 
&= \frac{1}{2} (e_i - e'_i)^T L^{-1} (e_i - e'_i) \\
&= \frac{1}{2} (e_i - e'_i)^T M^{+} (e_i - e'_i) \\
&\le \frac{1}{2} (e_i - e_{n + 1})^T M^+ (e_i - e_{n + 1}) + \frac{1}{2} (e'_i - e_{n + 1})^T M^+ (e_i' - e_{n + 1}) \\
&\le \frac{1}{2} \frac{1}{M_{i,n+1}} + \frac{1}{2} \frac{1}{M_{i',n+1}} \le \frac{1}{|\Theta_{ij}|}.
\end{align*}
\end{proof}
\begin{lemma}\label{lem:bound-after-conditioning}
Suppose that $i$ is a non-isolated node with $d$ neighbors in an SDD GGM. Then for at least one neighbor $j$ it holds that
\[ \Var(X_i | X_j) \le \frac{4d}{\Theta_{ii}} \]
\end{lemma}
\begin{proof}

We establish the following dichotomy: either $\Var(X_i)$ is already small, or if it is large then there is a $j$ s.t. $1/|\Theta_{ij}|$ is small so $\Var(X_i | X_j)$ is small. Observe by Cauchy-Schwartz that
\begin{align*}
\Theta_{ii} \Var(\E[X_i | X_{\sim i}]) 
= \Theta_{ii} \Cov(\E[X_i | X_{\sim i}], \E[X_i | X_{\sim i}])
&= \sum_j -\Theta_{ij} \Cov(\E[X_i | X_{\sim i}], X_j) \\
&\le \sum_j |\Theta_{ij}| \sqrt{\Var(\E[X_i | X_{\sim i}]) \Var(X_j)}
\end{align*}
so
\begin{align*}
 \Theta_{ii} \sqrt{\Var(\E[X_i | X_{\sim i}])} 
 \le \sum_j |\Theta_{ij}| \sqrt{\Var(X_j)} 
 \le \sum_j |\Theta_{ij}| \sqrt{\Var(X_i) + 1/|\Theta_{ij}|} 
 &\le \sqrt{\Var(X_i)} \sum_j |\Theta_{ij}| + \sum_j \sqrt{|\Theta_{ij}|}  \\
 &\le \sqrt{\Var(X_i)} \sum_j |\Theta_{ij}| + \sqrt{d \Theta_{ii}}
\end{align*} 
where in the second inequality we used Lemma~\ref{lem:sdd-smooth-variance}, in the third inequality we used $\sqrt{a + b} \le \sqrt{a} + \sqrt{b}$, and in the fourth inequality we used Cauchy-Schwartz and the SDD assumption.

Suppose that $\Var(\E[X_i | X_{\sim i}]) > 4d/\Theta_{ii}$. Then by subtracting $d\sqrt{\Theta_{ii}}$ from both sides we see
\[ \frac{1}{2} \Theta_{ii} \sqrt{\Var(\E[X_i | X_{\sim i}])} \le \sqrt{\Var(X_i)} \sum_j |\Theta_{ij}| \le \sqrt{\Var(X_i)} d \max_j |\Theta_{ij}| \]
so using that $\Var(\E[X_i | X_{\sim i}]) = \Var(X_i) - 1/\Theta_{ii} \ge \Var(X_i)/2$ under our assumption, we find
\[ \frac{\Theta_{ii}}{4d} \le \frac{\Theta_{ii}}{2d} \sqrt{\frac{\Var(\E[X_i | X_{\sim i}])}{\Var(X_i)}} \le \max_j |\Theta_{ij}|. \]
Let $j$ be the maximizer, then from Lemma~\ref{lem:bound-after-conditioning-ij} we find $\Var(X_i | X_j) \le \frac{1}{|\Theta_{ij}|} \le \frac{4d}{\Theta_{ii}}$, assuming that $\Var(X_i) > 4d/\Theta_{ii}$. Otherwise, by the law of total variance we know $\Var(X_i | X_j) \le \Var(X_i) \le 4d/\Theta_{ii}$. 
\end{proof}
The following example shows that the assumption that the matrix is SDD (or walk-summable) is necessary for the previous Lemmas to be true:
\begin{example}[Failure of Lemma~\ref{lem:bound-after-conditioning-ij} in Non-SDD GGM]\label{example:big-cancellation}
Consider for $\kappa$ fixed and $C$ large
\[ \Theta := \begin{bmatrix}
1 & C & -C \\
C & C^2/\kappa^2 & -C^2/\kappa^2 + 1 \\
-C & -C^2/\kappa^2 + 1 & C^2/\kappa^2
\end{bmatrix} \]
We can verify that as $C \to \infty$ that the variances (i.e. diagonal of $\Theta^{-1}$) remain $\Theta(1)$ and the matrix is positive definite; furthermore this model is $\kappa$-nondegenerate. However, even after conditioning out the first node, the variance of the second (and third) node remains $\Omega(1) \gg  1/C$.
\end{example}

\section{Estimating changes in conditional variance}
As alluded to before, our algorithms rely on estimating (differences of) conditional variances $\Var(X_i | X_S)$. The classical approach for estimating them is to solve a linear regression problem trying to predict $X_i$ from $X_S$. As we are working in a sample-starved regime and deal with ill-conditioned matrices, we require very fine grained results about such estimates. We collect such results in this section. 

For the analysis of Algorithm~\textsc{HybridMB} we only need the basic facts from Section 4.1; for the analysis of Algorithm~\textsc{GreedyAndPrune} the key additional fact we need is encapsulated as Lemma~\ref{lem:estimating-variance-decrement-v} in Section 4.3 below; finally, for the analysis of the Algorithm~\textsc{SearchAndValidate} we will also directly use the results stated in Section 4.2.
\subsection{Background: Fixed Design Linear Regression}\label{sec:fixed-design}
In this section we recall the standard model for linear regression with Gaussian noise and the usual ordinary least squares estimator and some classical facts about it. See Chapter 14 of \cite{keener2011theoretical} for a reference.
\begin{defn}[Fixed design regression with Gaussian noise]
 The (well-specified) \emph{fixed design regression} model is specified by an unknown parameter $w \in \mathbb{R}^k$, known \emph{design matrix} $\mathbbm{X} : m \times k$ with $m > k$ and observations
\[ \mathbbm{Y} = \mathbbm{X} w + \Xi \]
where $\Xi \sim N(0, \sigma^2 I)$. In other words, $\mathbbm{Y} \sim N(\mathbbm{X} w, \sigma^2 I)$.
\end{defn}
\begin{defn}[Ordinary Least Squares (OLS) Estimator]\label{defn:ols}
The OLS estimator for $w$ in the fixed design regression model is the minimizer of
\[ \min_w \|\mathbbm{Y} - \mathbbm{X} w\|_2^2 \]
explicitly given by
\[ \hat{w} := (\mathbbm{X}^T\mathbbm{X})^{-1} \mathbbm{X}^T \mathbbm{Y} \]
assuming that $\mathbbm{X}$ has maximal column rank. 
The corresponding estimator for $\sigma$ is given by
\[ \hat{\sigma}^2 := \frac{1}{m - k} \|\mathbbm{Y} - \mathbbm{X} \hat{w}\|_2^2. \]
\end{defn}

\begin{fact}[\cite{keener2011theoretical}]
\label{fact:ols-fixed-design}
Under the fixed design regression model with Gaussian noise, 
$\hat{w} \sim N(w, \sigma^2 (\mathbbm{X}^T\mathbbm{X})^{-1})$
and $\frac{(m - k)\hat{\sigma}^2}{\sigma^2} \sim \chi_{m - k}^2$ where $\chi^2_{m - k}$ denotes a $\chi^2$-distribution with $m - k$ degrees of freedom.  Furthermore, $\hat{w}$ and $\hat{\sigma}$ are independent.
\end{fact}

\begin{lemma} 
\label{lem:ols-sigmahat}
For any $\delta \in (0,1)$,
  \[ \Pr\left(\left|\frac{\hat{\sigma}^2}{\sigma^2} - 1\right| > 2\sqrt{\frac{\log(2/\delta)}{m - k}} + 2\frac{\log(2/\delta)}{m - k}\right) \le \delta. \]
\end{lemma}
\begin{proof}
Combine Fact~\ref{fact:ols-fixed-design} and 
and the concentration inequality from Lemma~\ref{lem:chi-squared-concentration}.
\end{proof}
We end with a geometric interpretation of the OLS coordinates which is analogous to Lemma~\ref{lem:gaussian-conditional-expectation}. In statistics this is known as the equivalence of the regression $t$-test and the 1-variable regression $F$-test \cite{keener2011theoretical}.
\begin{lemma}\label{lem:t-equals-f}
\[ \min_w \|\mathbbm{Y} - \mathbbm{X} w\|_2^2 - \min_{w : w_i = 0} \|\mathbbm{Y} - \mathbbm{X} w\|_2^2 = \frac{\hat{w}_i^2}{[(\mathbbm{X}^T\mathbbm{X})^{-1}]_{ii}} \]
\end{lemma}
\begin{proof}[Proof sketch]
Let $\mathbbm{X}_i$ be the $i$'th column of $\mathbbm{X}$. 
By the definition of the OLS estimate $\hat{w}$ and the Pythagorean theorem, the left hand side is equal to $\min_{w : w_i = 0} \|\mathbbm{X} \hat{w} - \mathbbm{X} w\|_2^2$. By another application of the Pythagorean theorem, this equals $\|\mathbbm{X}_i \hat{w}_i - \Proj_{V_i} \mathbbm{X}_i \hat{w}_i\|_2^2 = \hat{w}_i^2 \|\mathbbm{X}_i - \Proj_{V_i} \mathbbm{X}_i\|_2^2$ where $V_i$ is the subspace spanned by the columns of $\mathbbm{X}$ except for $i$. Finally $\|\mathbbm{X}_i - \Proj_{V_i} \mathbbm{X}_i\|_2^2 = \frac{1}{[(\mathbbm{X}^T\mathbbm{X})^{-1}]_{ii}}$ by applying Schur complement formulas.
\end{proof}

\subsection{Background: Random Design Linear Regression and Wishart Matrices}
Under fixed design, the matrix $\mathbbm{X}$ was considered to be a deterministic quantity. Random design (see e.g. \cite{hsu2012random} for references) corresponds to the case where the rows of $\mathbbm{X}$ are i.i.d. samples from some distribution, which fits the usual setup in statistical learning theory.
\begin{defn}[Random design linear regression with Gaussian covariates]
The random design linear regression model with Gaussian covariates with $m$ samples is given by a (typically unknown) covariance matrix $\Sigma : k \times k$, i.i.d. samples $X^{(1)},\ldots,X^{(m)} \sim N(0,\Sigma)$ and corresponding observations
\begin{equation}\label{eqn:random-design-setup}
Y^{(i)} = \langle X^{(i)}, w \rangle + \xi^{(i)},\;\;\; i = 1,\ldots,m 
\end{equation}
where each $\xi^{(i)} \sim N(0,\sigma^2)$ is independent noise. (The assumption that $\xi^{(i)}$ is independent is referred to as the model being \emph{well-specified}.)
\end{defn}
The OLS estimator is defined as before in Definition \ref{defn:ols} where the rows of the design matrix $\mathbbm{X}$ are the samples $X_1,\ldots,X_m$ and $\mathbbm{Y} = (Y^{(i)})_{i = 1}^m$. From \eqref{fact:ols-fixed-design} we still have that for fixed $X_1,\ldots,X_m$ (i.e. considering only the randomness over $\xi_1,\ldots,\xi_m$)
\[ \hat{w}_{OLS} \sim N(w, \sigma^2 (\mathbbm{X}^T\mathbbm{X})^{-1}). \]
Therefore reasoning about the OLS estimator under random design can be reduced to understanding the random matrix $\mathbbm{X}^T\mathbbm{X}$, which is referred to as a \emph{Wishart matrix} (with $m$ degrees of freedom). We recall here a standard concentration inequality for Wishart matrices when $\Sigma = I$. (This inequality generalizes to the sub-Gaussian case and we have specialized it for simplicity.)
\begin{theorem}[Theorem 4.6.1, \cite{vershynin2018high}]\label{thm:covariance-concentration}
Suppose that $X^{(1)},\ldots,X^{(m)} \sim N(0,I)$ are independent Gaussian random vectors in $\mathbb{R}^k$, then
\[ \left\|\frac{1}{m} \sum_{i= 1}^m X^{(i)} (X^{(i)})^T - Id\right\| \le C_1 \left(\sqrt{\frac{k}{m}} + \sqrt{\frac{\log(2/\delta)}{m}}\right) \]
for some absolute constant $C_1 > 0$,
with probability at least $1 - \delta$. 
\end{theorem}
This leads to a multiplicative guarantee for general Wishart matrices:
\begin{lemma}\label{lem:wishart-concentration}
Suppose $\epsilon \in (0,1/2)$ and $\delta > 0$. Then for any $m$ such that $\epsilon \le  C_1 \left(\sqrt{\frac{k}{m}} + \sqrt{\frac{\log(2/\delta)}{m}}\right)$ and $X^{(1)},\ldots,X^{(m)} \sim N(0,I)$ we have that
\[ (1 - \epsilon) \Sigma \preceq \frac{1}{m} \sum_i X_i X_i^T \preceq (1 + \epsilon) \Sigma  \]
with probability at least $1 - \delta$.
\end{lemma}
\begin{proof}
This is equivalent to showing that
\[ (1 - \epsilon) I \preceq \frac{1}{m} \sum_i \Sigma^{-1/2} X^{(i)} (\Sigma^{-1/2} X^{(i)})^T \preceq (1 + \epsilon) I \]
since the PSD ordering is preserved under matrix congruence. The above follows from applying Theorem~\ref{thm:covariance-concentration} to $\bar{X}^{(i)} = \Sigma^{-1/2} X^{(i)}$.
\end{proof}
\begin{defn}
Given i.i.d. mean-zero random vectors $X^{(1)},\ldots,X^{(m)}$ the \emph{empirical covariance matrix} is
\[ \widehat{\Sigma} := \frac{1}{m} \sum_i X^{(i)} (X^{(i)})^T. \]
\end{defn}
\subsection{Estimating changes in conditional variance}

We are now ready to state what we need for estimating changes in conditional variance. Recall the basic setup: Given samples from $X$ from a GGM at various stages in our algorithm we use estimates for conditional variances of the form $\Var(X_i | X_S)$ by regressing $X_i$ against $X_S$. What we really we need are not actual values of $\Var(X_i | X_S)$ but to find a variable $j \notin S$ that gives non-trivial (or even \emph{most}) advantage in predicting $X_i | X_{S \cup \{j\}}$. So we need to quantify the relative advantage of including an additional variable $j$ on top of $S$. 

We can abstract the above in the regression setting as follows: Given samples for regression $(X,Y)$, and an index $j$ check if $\Var(Y|X) = \Var(Y|X_{\sim j})$. That is, whether including feature $x_j$ gives non-trivial advantage in regression. This is akin to the classical \emph{regression $t$-test} in statistics (see \cite{keener2011theoretical}) used to test the null hypothesis that $w_i = 0$ in a linear regression problem. 


In the greedy steps in our learning algorithm, we will need to not only find a feature which has a nonzero value for predicting $Y$, but in fact we want to find one of the most predictive features.  
We do so by exploiting what is known as a \emph{non-central} $F$-statistic \cite{keener2011theoretical}. The following lemma quantifies the \emph{usefulness} of a particular coordinate for estimating $Y$. 
Crucially, this Lemma shows we can estimate the (normalized) change in conditional variance much more accurately than we can actually estimate the individual conditional variances. Note that by Lemma~\ref{lem:t-equals-f} that the term which appears in the Lemma, $\frac{|\hat{w}_j|^2}{(\hat \Sigma^{-1})_{jj}}$, also equals the difference in squared loss over the data between the OLS estimator constrained to $w_j = 0$ and the unconstrained OLS estimator.
\begin{lemma}\label{lem:estimating-variance-decrement}
Consider the Gaussian random design regression setup \eqref{eqn:random-design-setup}, fix $j \in \{1,\ldots,k\}$ and let
\[ \gamma := \frac{\Var(Y | X_{\sim j}) - \Var(Y | X)}{\Var(Y | X)} \]
where $X_{\sim j} = (X_i)_{i \ne j}$.
We have
\[ \left|\frac{|\hat{w}_j|}{\hat{\sigma} \sqrt{(\hat \Sigma^{-1})_{jj}}} - \sqrt{\gamma}\right| \le \sqrt{\frac{4 \log(4/\delta)}{m}} + \sqrt{\frac{\gamma}{64}} \]
and
\[ \left|\frac{|\hat{w}_j|}{\sigma \sqrt{(\hat \Sigma^{-1})_{jj}}} - \sqrt{\gamma}\right| \le \sqrt{\frac{2 \log(4/\delta)}{m}} + \sqrt{\frac{\gamma}{64}} \]
with probability at least $1 - \delta$ as long as $m \ge m_0 = O(k + \log(4/\delta))$.
\end{lemma}
\begin{proof}
We prove this result directly. Alternatively and essentially equivalently, one could derive a similar result by using classical results in the fixed design regression setting for non-central F-statistics (Theorem 14.11 of \cite{keener2011theoretical}, see also Section~\ref{sec:info-optimal} below) and then analyzing their behavior under random design using matrix concentration.

Recall from Lemma~\ref{lem:gaussian-conditional-expectation} (applied for fixed $X_S$ and then taking expectations) that
\[ \E[Y | X] = \E[Y | X_{\sim j}] + \frac{\Cov(Y,X_j | X_{\sim j})}{\Var(X_j | X_{\sim j})}(X_j - \E[X_j | X_{\sim j}]) \]
and that
\[ \Var(Y | X_{\sim j}) - \Var(Y | X) = \frac{\Cov(Y,X_j | X_{\sim j})^2}{\Var(X_j | X_{\sim j})} \]
so 
\begin{equation}\label{eqn:t-equals-f}
w_j^2 \Var(X_j | X_{\sim j}) = \Var(Y | X_{\sim j}) - \Var(Y | X). 
\end{equation}
i.e. $\frac{w_j^2}{\sigma^2 (\Sigma^{-1})_{jj}} = \gamma$.
We know that for fixed $X$, over the randomness of $\xi$ we have $\hat{w}_{OLS} \sim N(w, \frac{\sigma^2}{m} \hat{\Sigma}^{-1})$ by Fact~\ref{fact:ols-fixed-design},
so
\[ \frac{\hat{w}_j}{\sigma \sqrt{(\hat{\Sigma}^{-1})_{jj}}} \sim N\left(\frac{w_j}{\sigma \sqrt{(\hat{\Sigma}^{-1})_{jj}}}, \frac{1}{m}\right). \]
Using that $(\Sigma^{-1})_{jj} = \frac{1}{\Var(X_j | X_S)}$, $\sigma = \sqrt{\Var(Y | X)}$, and $\gamma = \frac{\Var(Y | X_{\sim j}) - \Var(Y | X)}{\Var(Y | X)}$ and \eqref{eqn:t-equals-f} we find
\[ \frac{\hat{w}_j}{\sigma \sqrt{(\hat{\Sigma}^{-1})_{jj}}} \sim N\left(\pm \sqrt{\gamma \frac{(\Sigma^{-1})_{jj}}{(\hat{\Sigma}^{-1})_{jj}}}, \frac{1}{m}\right) \]
where the sign is the sign of $w_j$. Applying $||a| - |b|| \le |a - b|$ and the Gaussian tail bound over the randomness of $\hat{w}$ we find
\[ \Pr\left(\left|\frac{|\hat{w}_j|}{\sigma \sqrt{(\hat{\Sigma}^{-1})_{jj}}} - \sqrt{\gamma \frac{(\Sigma^{-1})_{jj}}{(\hat{\Sigma}^{-1})_{jj}}}\right| > t\right) \le  \Pr\left(\left|\frac{\hat{w}_j}{\sigma \sqrt{(\hat{\Sigma}^{-1})_{jj}}} \mp \sqrt{\gamma \frac{(\Sigma^{-1})_{jj}}{(\hat{\Sigma}^{-1})_{jj}}}\right| > t\right) \le 2e^{-mt^2/2}. \]
Applying Lemma~\ref{lem:ols-sigmahat} gives
\[ \left|\frac{\hat{\sigma}}{\sigma} - 1\right| \le 2\sqrt{\frac{\log(4/\delta)}{m - k - 1}} + 2\frac{\log(4/\delta)}{m - k - 1} \]
with probability at least $1 - \delta/2$. Therefore as long as $m \ge m_1 = O(k + \log(4/\delta))$ we have $\frac{\hat{\sigma}}{\sigma} \in (7/8,9/8)$. Taking $t = \sqrt{2\log(4/\delta)/m}$ we have
\[ \left|\frac{|\hat{w}_j|}{\hat{\sigma} \sqrt{(\hat \Sigma^{-1})_{jj}}} - \sqrt{\gamma} \right|
\le \sqrt{\frac{\sigma}{\hat{\sigma}}} \left|\frac{|\hat{w}_j|}{\sigma \sqrt{(\hat \Sigma^{-1})_{jj}}} - \sqrt{\gamma \frac{\hat{\sigma}}{\sigma}} \right| + \sqrt{\gamma} \left|1 - \sqrt{ \frac{(\hat{\Sigma}^{-1})_{jj}}{(\Sigma^{-1})_{jj}}}\right|
\le \sqrt{\frac{4 \log(4/\delta)}{m}} + \sqrt{\frac{\gamma}{64}} \]
applying Lemma~\ref{lem:wishart-concentration} and requiring $m \ge m_2 = O(k + \log(4/\delta))$, with probability at least $1 - \delta$. A simpler variant of this argument gives the result for $\frac{|\hat{w}_j|}{\sigma \sqrt{(\hat \Sigma^{-1})_{jj}}}$ as well.
\end{proof}
In our analysis we will often need to estimate multiplicative changes in a quantity of the form $\Var(Y | X_{\sim j}) -  V$ (where e.g. $V = \Var(Y | X, X')$ for some $X'$) so we will use the following variant of the previous Lemma:
\begin{lemma}\label{lem:estimating-variance-decrement-v}
Consider the Gaussian random design regression setup \eqref{eqn:random-design-setup}, fix $j \in \{1,\ldots,k\}$, let $V > 0$ be arbitrary s.t. $V < \Var(Y|X)$ and let
\[ \gamma' := \frac{\Var(Y | X_{\sim j}) - \Var(Y | X)}{\Var(Y | X_{\sim j}) - V} \]
where $X_{\sim j} = (X_i)_{i \ne j}$.
We have 
\[ \left|\sqrt{\frac{1}{\Var(Y | X_{\sim j}) - V}} \frac{|\hat{w}_j|}{\sqrt{(\hat \Sigma^{-1})_{jj}}} - \sqrt{\gamma'}\right| \le \sqrt{\frac{\Var(Y | X)}{\Var(Y | X_{\sim j}) - V} \cdot \frac{2 \log(4/\delta)}{m}} + \sqrt{\frac{\gamma'}{64}} \]
with probability at least $1 - \delta$ as long as $m \ge m_0 = O(k + \log(4/\delta))$.
\end{lemma}
\begin{proof}
This follows from Lemma~\ref{lem:estimating-variance-decrement} after multiplying through in the guarantee by $\sqrt{\gamma'/\gamma}$, using that $\sigma = \sqrt{\Var(Y | X)}$.
\end{proof}
\section{Learning all attractive GGMs efficiently}
\begin{defn}
We say that a GGM is \emph{attractive} (or \emph{ferromagnetic}) if $\Theta_{ij} \le 0$ for all $i \ne j$. (This is the same as requiring that $\Theta$ is an $M$-matrix.)
\end{defn}
\begin{lemma}
If $\Theta$ is the precision matrix of an attractive GGM, then there exists an invertible diagonal matrix $D$ with nonnegative entries such that $D \Theta D$ is a generalized Laplacian.
\end{lemma}
\begin{proof}
This follows immediately from Theorem~\ref{thm:rescaling}.
\end{proof}
A particularly important example of an attractive GGM is the \emph{discrete Gaussian free field} --- see \cite{sheffield2007gaussian} for a reference to this and the closely related literature on the \emph{continuum Gaussian free field}.
\begin{defn}
The \emph{discrete Gaussian free field} on a weighted graph $G$ with zero boundary conditions on $S$ is the GGM with $\Theta$ the Laplacian of $G$, after eliminating the rows and columns corresponding to the nodes in $S$.
\end{defn}
Without boundary conditions, the GFF should be translation invariant and so it does not exist as a probability distribution. One can approach this by taking the Laplacian and adding $\epsilon I$ to make it invertible, which gives a learnable model that is arbitrarily poorly conditioned.
\begin{example}[Gaussian simple random walk] Consider the discrete Gaussian free field on a path of length $n$ with zero boundary condition on the first node. This process is the same as a simple random walk with $N(0,1)$ increments. That is the resulting distribution is of the form $(X_1,\ldots,X_n)$ where $X_i = \sum_{j \leq i} \eta_j$ for independent and identical $\eta_j \sim N(0,1)$. From the GFF perspective, we can think of this as a discretization of Brownian motion (the one-dimensional (continuum) Gaussian free field).
\end{example}
\begin{remark}\label{rmk:attractive-ggm-is-gff}
Every attractive GGM can be realized from a Gaussian Free Field on a weighted graph in the following way: given an attractive GGM, first rescale the coordinates using the above Lemma so that it is a generalized Laplacian. Then, by adding one node to the model we can make the precision matrix into a standard Laplacian on some weighted graph, and conditioning out the added node recovers the original precision matrix.
\end{remark}

Our main theorem of this section is a sample-efficient algorithm for learning attractive GGMs:

\begin{theorem}\label{thm:greedy-ferromagnetic}
Fix a $\kappa$-nondegenerate attractive GGM.
Algorithm \textsc{GreedyAndPrune} returns the true neighborhood of every node $i$ with probability at least $1 - \delta$ for $\nu = \kappa^2/\sqrt{32}, K = 64d\log(4/\kappa^2) + 1$ as long as the number of samples $m \ge m_1$ for $m_1 = O((1/\kappa^2)(K\log(n) + \log(4/\delta)))$. The combined run-time (over all nodes) of the algorithm is $O(K^3 m n^2)$. 
\end{theorem}
Note that the above immediately implies Theorem \ref{thm:greedy-ferromagnetic-informal}. 

As mentioned in the introduction, Algorithm \textsc{GreedyAndPrune} learns the neighborhood of a node by doing greedy forward selection to minimize the conditioned variance, and then doing pruning to remove non-neighbors from the candidate neighborhood. The greedy forward selection step is known in the compressed sensing literature as \emph{Orthogonal Matching Pursuit} (OMP) (see e.g. \cite{tropp2007signal}). We give a description of the OMP algorithm in the general setting of Section~\ref{sec:fixed-design} below, along with pseudocode for \textsc{GreedyAndPrune}.
\begin{figure}
\fbox{\begin{minipage}{\textwidth}
\vspace{.1cm}
Algorithm \textsc{OrthogonalMatchingPursuit($T$):}
\begin{enumerate}
    \item Set $S_0 := \{\}$.
    \item For $t$ from $1$ to $T$:
    \begin{enumerate}
        \item Choose $j$ which minimizes
        \[ \min_{w \ :\  \supp(w) \subset S_{t - 1} \cup \{j\}} \|\mathbbm{Y} - \mathbbm{X} w\|_2^2 \]
        \item Set $S_t := S_{t - 1} \cup \{j\}$
    \end{enumerate}
    \item Return $S_T$.
\end{enumerate}
\vspace{.1cm}
\end{minipage}}
\end{figure}

\begin{figure}
\fbox{\begin{minipage}{\textwidth}
\vspace{.1cm}
Algorithm \textsc{GreedyAndPrune($i,\nu,T$):}
\begin{enumerate}
    \item  Run OMP for $T$ steps to predict $\mathbbm{X}_i$ from the other columns of $\mathbbm{X}$.
    \item Define $\hat{\Theta}_{ii}$ by  $1/\hat{\Theta}_{ii} = \widehat{\Var}(X_i | X_S)$.
    \item For $j \in S$:
    \begin{enumerate}
        \item Let $S' := S \setminus \{j\}$ and $\hat{w} := \hat{w}(i,S')$.
        \item If $\widehat{\Var}(X_i | X_{S'}) - \widehat{\Var}(X_i | X_S) < \nu/\hat{\Theta}_{ii}$, set $S := S'$.
    \end{enumerate}
    \item Return $S$.
\end{enumerate}
\vspace{.1cm}
\end{minipage}}
\end{figure}

\begin{remark}[Implementation: Merging neighborhoods]\label{rmk:merging-neighborhoods}
In order to return an actual estimate for the inverse precision matrix, we add in our implementation of \textsc{GreedyAndPrune} a merging step which includes an edge $(i,j)$ iff it is in the computed neighborhood of node $i$ and in the computed neighborhood of node $j$. Then to estimate the entries, we use OLS to predict node $X_i$ from its neighbors and estimate the conditional variance of $X_i$. We define $\hat{\Theta}_{ii}$ to be the inverse of the estimated conditional variance, and $-\hat{\Theta}_{ij}/\hat{\Theta}_{ii}$ to be the OLS coefficient. Finally, we symmetrize $\hat{\Theta}$ by picking the smaller of absolute norm between $\hat{\Theta}_{ij}$ and $\hat{\Theta}_{ji}$; the same step is used in CLIME \cite{cai2011constrained}.
\end{remark}
\subsection{Proof of supermodularity}
As a first step toward proving Theorem \ref{thm:greedy-ferromagnetic}, we first show that the conditional variance function is supermodular. 
\begin{defn}
Given a universe $U$, a function $f : 2^U \to \mathbb{R}$ is \emph{supermodular} if for any $S \subset T$,
\[ f(S) - f(S \cup \{j\}) \ge f(T) - f(T \cup \{j\}). \]
(This is the same as saying $-f$ is \emph{submodular}.)
\end{defn}
Supermodularity of the conditional variance of a node in the GFF (and hence, by using the reduction from Remark~\ref{rmk:attractive-ggm-is-gff}, all attractive GGMs) was previously shown independently in \cite{mgs2013,mahalanabis2012subset} using two different methods. The proof in \cite{mgs2013} is algebraic using the Schur complement formula, whereas the proof in \cite{mahalanabis2012subset} converts the problem into one about electrical flows and argues via Thomson's principle. We give a third different proof which has the benefit of being transparent and using only basic linear algebra.

\ignore{
The proof is based on the following lemma, which generalizes the random-walk interpretation of the Laplacian/GFF to attractive GGMs. 
\begin{lemma}\label{lem:walk-expansion}
  In a ferromagnetic GGM with $\Theta_{ii} = 1$ for all $i$, if we write $\Theta = I - A$ where $diag(A) = 0$ then the covariance matrix $\Sigma = \Theta^{-1}$ satisfies
  \[ \Sigma = \sum_{k = 0}^{\infty} A^k. \]
\end{lemma}
\begin{proof}
By ferromagneticity of the GGM, $A$ has all nonnegative entries. From the Perron-Frobenius Theorem and from $0 \prec \Theta = I - A$ we see that $\|A\| = \lambda_{max}(A) < 1$.
Then
\[ \Sigma = (I - A)^{-1} = \sum_{k = 0}^{\infty} A^k \]
using the identity $1/(1 - x) = \sum_{i = 0}^{\infty} x^i$ valid when $|x| < 1$.
\end{proof}}
\begin{theorem}\label{thm:supermodular}
For any node $i$ in a ferromagnetic GGM, $\Var(X_i | X_S)$ is a monotonically decreasing, supermodular function of $S$.
\end{theorem}
\begin{proof}
By rescaling we may assume w.l.o.g. that $\Theta_{ii} = 1$ for all $i$. Define $\Theta_S$ to be the precision matrix corresponding to conditioning $S$ out (i.e. $\Theta$ with the rows and columns of $S$ removed), and $\Sigma_S = \Theta_S^{-1}$. Then, if we write $\Theta_S = I - A_S$, by Neumann series formula (as $\Theta_S \succ 0$, $\|A_S\| < 1$ using Perron-Frobenius), we see
\begin{equation}\label{eqn:attractive-walk-expansion}
\Sigma_S = (I - A_S)^{-1} = \sum_{k = 0}^{\infty} A_S^k. 
\end{equation}
Writing this out explicitly for $(\Sigma_S)_{i,i}$ gives
\begin{equation}\label{eqn:variance-loops}
\Var(X_i | X_S) = \sum_{k = 0}^{\infty} \sum_{v_1,\ldots,v_k \notin S}
(-\Theta_{i v_1}) \cdots (-\Theta_{v_k i}), 
\end{equation}
where the $k = 0$ term in the sum is interpreted to be $1$,
so $\Var(X_i | X_S)$ is a nonnegative weighted sum over walks avoiding $S$ and returning to $i$ in the final step. The above expression is clearly monotonically increasing in $S$ as all off-diagonal entries of $\Theta$ are negative (and also follows from law of total variance); to verify supermodularity, we just need to check that
\[ \Var(X_i | X_S) - \Var(X_i | X_{S \cup \{j\}}) =  \sum_{k = 0}^{\infty} \sum_{\substack{v_1,\ldots,v_k \notin S,\\ j \in \{v_1,\ldots,v_k\}}} (-\Theta_{i v_1}) \cdots (-\Theta_{v_k i}) \]
is a monotonically decreasing function of $S \subseteq [n]\setminus \{i,j\}$,
but this is clear once we apply \eqref{eqn:variance-loops} as the set of cycles that are eliminated from the sum by adding $j$ only shrinks as we increase $S$.
\end{proof}

Supermodularity of the conditional variance has the following useful consequence which will later be useful in showing that the greedy algorithm makes non-trivial progress in each step. 
\begin{lemma}\label{lem:supermodularity-consequence}
For any node $i$ in a ferromagnetic GGM, if $S$ is a set of nodes that does not contain $i$ or all neighbors of $i$, and $T$ is the set of neighbors of $i$ not in $S$, then there exists some node $j \in T$ such that
\[ \Var(X_i | X_S) - \Var(X_i | X_{S \cup \{j\}}) \ge \frac{\Var(X_i | X_S) - 1/\Theta_{ii}}{|T|}\,. \]
\end{lemma}
\begin{proof}
This is a standard consequence of supermodularity -- we include the proof for completeness.

Consider adjoining the elements of $T$ to $S$ one at a time, and then apply supermodularity to show
\[ \Var(X_i | X_S) - \Var(X_i | X_{S \cup T}) \le  \sum_{j \in T} (\Var(X_i | X_S) - \Var(X_i | X_{S \cup \{j\}})) \le |T| \max_{j \in T} (\Var(X_i | X_S) - \Var(X_i | X_{S \cup \{j\}})). \]
Rearranging and using $\Var(X_i |  X_{S \cup T}) = 1/\Theta_{ii}$ (by the Markov property) gives the result.
\end{proof}
From \eqref{eqn:attractive-walk-expansion} we see immediately that the entries of the covariance $\Sigma$ of an attractive GGM are always nonnegative (this is why they are called attractive/ferromagnetic); we record this fact for future use.
\begin{lemma}[Griffith's inequality]\label{lem:griffiths}
In an attractive GGM, $\Cov(X_i,X_j) \ge 0$ for any $i,j$.
\end{lemma}
This fact is very well-known, holds for arbitrary ferromagnetic graphical models (i.e. not just Gaussian) and is referred to as \emph{Griffith's inequality}. See \cite{griffiths1969rigorous} for a more general proof.

\ignore{
\begin{remark}
It is possible to analyze the natural greedy learning algorithm using only the concentration inequality above (Lemma~\ref{lem:ols-sigmahat}) to bound the error between the true variance and conditional variance, however this will lead to a suboptimal dependence on $\kappa$ ($1/\kappa^4$ instead of $1/\kappa^2$). The issue is that the conditional variance estimates only converge at a slow $1/\sqrt{m}$ rate due to the Central Limit Theorem, whereas the correct analysis (below) instead relies on the fast $1/m$ rate for the OLS risk.
\end{remark}}
\subsection{Greedy Subset Selection in Attractive Models}
In this section we give a guarantee for \emph{subset selection} using OMP, by showing that after a small number of rounds OMP finds a set $S$ such that $\Var(X_i | X_S)$ is close to minimal. The sample complexity analysis is complicated by the fact that supermodularity holds at the level of the population loss (i.e. for an infinite amount of data) whereas it would be more convenient if it held for the empirical conditional variance, so we have to deal with both the regression noise and the randomness of the regressors.
First we prove the following lemma which gives a stronger version of Lemma~\ref{lem:kappa-variance} for ferromagnetic GGMs:
\begin{lemma}\label{lem:kappa-variance-ferromagnetic}
Fix $i$ a node in a $\kappa$-nondegenerate ferromagnetic GGM, and let $S$ be set
of nodes and let $T$ be the set of neighbors of $i$ not in $S$. Then
\[ \Var(X_i | X_S) \ge \frac{1 + |T|\kappa^2}{\Theta_{ii}} \]
\end{lemma}
\begin{proof}
By the law of total variance, Griffith's inequality (Lemma~\ref{lem:griffiths}), 
and the law of total variance again
\begin{align*}
 \Var(X_i | X_S) - \frac{1}{\Theta_{ii}}
= \Var(\E[X_i | X_{\sim i}] | X_S)
&= \Var(\sum_{j \in T} \frac{-\Theta_{ij}}{\Theta_{ii}} X_j | X_S)\\
&\ge \sum_{j \in T} \frac{\Theta_{ij}^2}{\Theta_{ii}^2} \Var(X_j | X_S) 
\ge \frac{1}{\Theta_{ii}} \sum_{j \in T} \frac{\Theta_{ij}^2}{\Theta_{ii} \Theta_{jj}} \ge \frac{|T| \kappa^2}{\Theta_{ii}}.
\end{align*}
\end{proof}
\begin{lemma}\label{lem:ferromagnetic-subset-selection}
  Suppose that $X$ is distributed according to an $\kappa$-nondegenerate ferromagnetic GGM and $i$ is a node of degree at most $d$. Let $\sigma^2 := \frac{1}{\Theta_{ii}}$ and $w^*_j = \frac{-\Theta_{ij}}{\Theta_{ii}}$ for all $j \ne i$.
  Then using $T$ rounds of OMP to predict $X_i$ given $X_{\sim i}$ from $m$ i.i.d. samples, we have that $\Var(\E[X_i | X_{\sim i}] | X_S) \le (1 - 1/2d)^{T - 1} \frac{8d}{\Theta_{ii}}$
    with probability at least $1 - \delta$
    provided that $m = \Omega((d + 1/\kappa^2)(T\log(n) + \log(2/\delta)))$.
\end{lemma}
\begin{proof}
We prove by induction that for every $1 \le t \le T$ that
\[ \Var(\E[X_i | X_{\sim i}] | X_{S_t}) \le (1 - 1/2d)^{t - 1} \frac{8d}{\Theta_{ii}}. \]
Note that by Lemma~\ref{lem:bound-after-conditioning} there exists a node $j$ such that $\Var(X_i | X_j) \le \frac{4d}{\Theta_{ii}}$. 
By taking a union bound, we may assume that:
\begin{enumerate}
    \item $\Var(X_i | X_{S_1}) \le \frac{8d}{\Theta_{ii}}$ using the above fact combined with Lemma~\ref{lem:ols-sigmahat} assuming that $m = \Omega(\log(n/\delta))$ to guarantee that the estimated conditional variances have small multiplicative error.
    \item For all subsets $U$ of $[n]$ of size at most $T$ and $j \in [n]$, applying Lemma~\ref{lem:estimating-variance-decrement-v} we have
    \[ \left|\frac{1}{\sqrt{\Var(X_i | X_{U \setminus \{j\}}) - 1/\Theta_{ii}}} \hat{R}(U,j) - \sqrt{\gamma'}\right| \le \sqrt{\frac{\Var(X_i | X_U)}{\Var(X_i | X_{U \setminus \{j\}}) - 1/\Theta_{ii}}} \sqrt{\frac{4 (T\log(n) + \log(12/\delta))}{m}} + \sqrt{\frac{\gamma'}{64}} \]
    where 
    \[ \gamma' = \gamma'(U,j) := \frac{\Var(X_i | X_{U \setminus \{j\}}) - \Var(X_i | X_U)}{\Var(X_i | X_{U \setminus \{j\}}) - 1/\Theta_{ii}} \]
    and
    \[ \hat{R}(U,j) := \frac{(\hat{w}_U)_j}{((\hat{\Sigma}_{U, U})^{-1})_{jj}} = \sqrt{\|\mathbbm{X}_i - \mathbbm{X} \hat{w}_{U}\|_2^2 - \|\mathbbm{X}_i - \mathbbm{X} \hat{w}_{U \setminus \{j\}}\|_2^2} \]
    using Lemma~\ref{lem:t-equals-f} in the last equality where $\hat{w}_U$ is the OLS
    estimate using only the coordinates in $U$. This holds assuming that $m = \Omega(T \log(4n) + \log(1/\delta))$.
\end{enumerate}
Before proceeding, we observe that
\begin{equation} \label{eqn:ferromagnetic-trick}
\sqrt{\frac{\Var(X_i | X_U)}{\Var(X_i | X_{U \setminus \{j\}}) - 1/\Theta_{ii}}} \le
\sqrt{\frac{\Var(X_i | X_{U \setminus \{j\}})}{\Var(X_i | X_{U \setminus \{j\}}) - 1/\Theta_{ii}}} \le \max(\sqrt{2}, \sqrt{2/d'\kappa^2})
\end{equation}
where $d'$ is the degree of node $i$ in the graph with the nodes in $U \setminus \{j\}$ removed, 
by the law of total variance (first inequality) and the following case analysis: either $\Var(X_i | X_{U \setminus \{j\}}) \ge 2/\Theta_{ii}$, in which case $\frac{\Var(X_i | X_{U \setminus j\}})}{\Var(X_i | X_{U \setminus \{j\}}) - 1/\Theta_{ii}} \le 2$, or $\Var(X_i | X_{U \setminus \{j\}}) \le 2/\Theta_{ii}$ in which case  $\frac{\Var(X_i | X_{U \setminus \{j\}})}{\Var(X_i | X_{U \setminus \{j\}}) - 1/\Theta_{ii}} \le 2/d'\kappa^2$ by Lemma~\ref{lem:kappa-variance-ferromagnetic}.

The first point above gives the base case for the induction.
By Lemma~\ref{lem:supermodularity-consequence}, if $\Var(\E[X_i | X_{\sim i}] | S_t) \ne 0$ then there exists a $k$ such that
\[ \gamma'(S_t \cup \{k\}, k) = \frac{\Var(\E[X_i | X_{\sim i}] | X_{S_t}) - \Var(\E[X_i | X_{\sim i}] | X_{S_t \cup \{k\}})}{\Var(\E[X_i | X_{\sim i}] | X_{S_t \cup \{k\}\}})} \ge \frac{1}{d'} \]
where (as above) $d'$ is the degree of $i$ in the set of non-neighbors of $S_t$. 
Combined with \eqref{eqn:ferromagnetic-trick} and $d' \le d$ we now see that the second guarantee above ensures that at every time $t$,  the $j$ selected by OMP (i.e. $j$ where $S_{t + 1} = S_t \cup \{j\}$) satisfies $\gamma'(S_t \cup \{j\}, j) \ge 1/2d$ as long as $m = \Omega((d + 1/\kappa^2)(T\log(n) + \log(12/\delta)))$. We therefore have that
\[ \Var(X_i | X_{S_t}) - 1/\Theta_{ii} \le (1 - 1/2d)(\Var(X_i | X_{S_{t - 1}}) - 1/\Theta_{ii}) \]
for all $1 < t \le T$, which combined with the induction hypothesis gives the result (using that $\Var(X_i | X_{S_t}) - 1/\Theta_{ii} = \Var(\E[X_i | X_{\sim i}] | X_{S_t})$ by law of total variance).
\end{proof}
\subsection{Structure Recovery for Attractive GGMs}
To give a final result for structure recovery, we show how to combine the previous analysis of greedy forward selection with a simple analysis of pruning (backward selection).
\begin{lemma}\label{lem:node-recovery-ferromagnetic}
Let $i$ be a node of degree at most $d$ in a $\kappa$-nondegenerate attractive GGM.
Fix $\delta > 0$ and suppose that $m = \Omega((d + 1/\kappa^2)(T\log(n) + \log(2/\delta)))$ where $T = \Theta(d\log(2d/\kappa^2))$. Then with probability at least $1 - \delta$, the neighborhod of node $i$ is correctly recovered by Algorithm \textsc{GreedyAndPrune} with $\nu = \Theta(\kappa^2)$.
\end{lemma}
\begin{proof}
By Lemma~\ref{lem:ferromagnetic-subset-selection} with $T = 1 + 2d\log(16d/\kappa^2)$, with probability at least $1 - \delta/2$ we have that $\Var(\E[X_i | X_{\sim i}] \mid X_S) \le \kappa^2/2$ where $S$ is the set returned by OMP as long as $m = \Omega((d + 1/\kappa^2)(T\log(n) + \log(24/\delta)))$. From Lemma~\ref{lem:kappa-variance} we see this implies that $S$ contains the true neighborhood of node $i$.

We now analyze the pruning step for any $S$ which is a superset of the true neighborhood of size at most $T$. By Lemma~\ref{lem:kappa-variance} and the Markov property, we know that if $j$ is a true neighbor then $\gamma(S,j) \ge \kappa^2$, and otherwise $\gamma(S,j) = 0$. Applying Lemma~\ref{lem:estimating-variance-decrement} and taking the union bound over the at most $n^T$ possible sets, we find that exactly the true edges are kept with probability at least $1 - \delta/2$ as long as $m = \Omega((T\log(n) + \log(8/\delta))/\kappa^2)$.  Therefore the entire neighborhood recovery succeeds with probability at least $1 - \delta$.
\end{proof}
\begin{theorem}
  Let $X$ be distributed according to a $\kappa$-nondegenerate GGM on $n$ nodes with maximum degree $d$. Fix $\delta > 0$, then with probability at least $1 - \delta$ Algorithm \textsc{GreedyAndPrune} run at every node with $T = \Theta(d\log(2d/\kappa^2))$ and $\nu = \Theta(\kappa^2)$ successfully recovers the true graph as long as $m = \Omega((1/\kappa^2)(d\log(2d/\kappa^2) + \log(2/\delta))\log(n))$.
\end{theorem}
\begin{proof}
This follows from Lemma~\ref{lem:node-recovery-ferromagnetic} by taking the union bound over the $n$ nodes and recalling from Lemma~\ref{lem:d-bounded-by-kappa} the bound $d \le 1/\kappa^2$.
\end{proof}
\begin{remark}[Input specification]
In the description of the algorithms throughout this paper, we assume we have access to i.i.d. samples from the distribution. However, it is straightforward to verify that the algorithms only depend on the empirical covariance matrix, and can be run given only the empirical covariance matrix in polynomial time.
\end{remark}
\section{Information-theoretic optimal learning of attractive GGMs}\label{sec:info-optimal}
In this section we give an $O(n^d)$ time algorithm for recovering attractive GGMs which matches the information-theoretic lower bounds up to constants, improving the result of the previous section at the cost of computational efficiency. 

\subsection{Noncentral F-statistics}
In the analysis of the $O(n^d)$ time algorithm, we will need to compare empirical variances between predictors supported on very different sets of variables. In comparison, in the analysis of greedy methods we only needed to consider adding or removing a single variable at a time. In order to handle the new setting, we recall the definition of noncentral F-statistics and their connection to fixed design regression.
\begin{defn}
Suppose $Z_1 \sim N(\delta,1)$ and for $j > 1$, $Z_j \sim N(0,1)$ with $Z_1,\ldots,Z_m$ independent. Then we write $\sum_i Z_i \sim \chi^2_m(\delta^2)$ where $\chi^2_m(\delta^2)$ is the \emph{noncentral chi-square distribution} with noncentrality parameter $\delta^2$ and $m$ degrees of freedom.
\end{defn}
\begin{defn}
If $V \sim \chi^2_k(\delta^2)$ and $W \sim \chi^2_m$ is independent of $V$, then we write
\[ \frac{V/k}{W/m} \sim F_{k,m}(\delta^2) \]
where $F_{k,m}(\delta^2)$ is the \emph{noncentral F-distribution} with degrees of freedom $k$ and $m$ and noncentrality parameter $\delta^2$.
\end{defn}
\begin{theorem}[Theorem 14.11 of \cite{keener2011theoretical}]\label{thm:noncentral-f-test}
In the (Gaussian) fixed design regression model (Section~\ref{sec:fixed-design}), let $H$ be a $q$-dimensional subspace of $\mathbb{R}^k$. Define
\[ T := \frac{m - k}{k - q} \frac{\|\mathbbm{Y} - \mathbbm{X} \hat{w}_0\|^2 - \|\mathbbm{Y} - \mathbbm{X} \hat{w}\|^2}{\|\mathbbm Y - \mathbbm{X} \hat{w}\|^2} = \frac{m - k}{k - q} \frac{\|\mathbbm{X} \hat{w} - \mathbbm{X} \hat{w}_0\|^2}{\|\mathbbm Y - \mathbbm{X} \hat{w}\|^2} \]
where $\hat{w}$ is the unrestricted OLS estimator and $\hat{w}_0$ is
the least squares estimator constrained to be inside of subspace $H$. (The second equality holds by the Pythagorean theorem.)
Then $T \sim F_{k - q,m - k}(\gamma)$ where
\[ \gamma := \frac{\min_{w_0 \in H_0} \|\mathbbm{X} (w - w_0)\|^2}{\sigma^2}. \]
More specifically, $\frac{1}{\sigma^2} \|\mathbbm{Y} - \mathbbm{X} \hat{w}\|^2 \sim \chi^2_{m - k}$ and $\frac{1}{\sigma^2} \|\mathbbm{X} \hat{w} - \mathbbm{X} \hat{w}_0\|^2 \sim \chi^2_{k - q}(\gamma)$ and these random variables are independent.
\end{theorem}
We also recall a convenient concentration inequality for noncentral $\chi^2$-distributed random variables: 
\begin{lemma}[Lemma 8.1 of \cite{birge2001alternative}]\label{lem:noncentral-chi-squared-concentration}
Suppose that $V \sim \chi_m^2(\delta^2)$. Then
\[ \Pr(V \ge (m + \delta^2) + 2\sqrt{(m + 2\delta^2) t} + 2t) \le e^{-t} \]
and
\[ \Pr(V \le (m + \delta^2) - 2\sqrt{(m + 2\delta^2)t}) \le e^{-t}. \]
\end{lemma}
\subsection{Structure learning by $\ell_0$-constrained least squares}
We perform structure recovery by, for every node $i$, performing several $\ell_0$-constrained regressions and pruning the result.
In the context of learning Gaussian graphical models, some algorithms in a similar spirit referred to as SLICE and DICE were proposed in \cite{misra18} and they proved a sample complexity bound of $O(d/\kappa^2 \log(n))$ for the more sample-efficient method, DICE. We show our estimator \textsc{SearchAndValidate} actually achieves optimal sample complexity $O((1/\kappa^2) \log(n))$ in the setting of attractive GGMs, and always achieves a sample complexity of $O((d/\kappa^2) \log(n))$ which gives a faster algorithm with the same sample complexity as DICE from \cite{misra18}, which has a slower runtime of $O(n^{2d + 1})$. (It matches the runtime guarantee for SLICE in \cite{misra18}, which has a worse sample complexity guarantee.)

In Algorithm~\textsc{SearchAndValidate}, the key step is performing $\ell_0$-constrained regression to predict $X_i$; the loop in step 2 is required only because we do not know a priori the exact degree of node $i$, only an upper bound. With high probability, the support of one of the $w_{d'}$ will equal the exact neighborhood of node $i$, and then a straightforward validation procedure in step 3 (which uses a similar idea to Algorithm~\textsc{DICE} in \cite{misra18}) allows us to identify the correct $w_{d'}$ successfully.
\begin{figure}
\fbox{\begin{minipage}{\textwidth}
\vspace{.1cm}
Algorithm \textsc{SearchAndValidate}(i,d,$\nu$):
\begin{enumerate}
    \item We assume the data has been split into two equally sized sample sets $1$ and $2$. Let $\hat{\E}_1$ and $\hat{\E}_2$ denote the empirical expectation over these two sets and define $\widehat{\Var}_2$ similarly.
    \item For $d'$ in $0$ to $d$:
    \begin{enumerate}
        \item Find $w_{d'}$ minimizing
        \[ \min_{w : w_i = 0, |\supp(w)| \le d'} \hat{\E}_1[(X_i - w_{d'} \cdot X)^2] \]
    \end{enumerate}
    \item For $d'$ in $0$ to $d$ (outer loop):
    \begin{enumerate}
        \item For $d''$ in $0$ to $d$ except $d'$ (inner loop):
        \begin{enumerate}
            \item Let $S_{d',d''} := \supp(w_{d'}) \cup \supp(w_{d''})$.
            \item For $j$ in $\supp(w_{d''}) \setminus \supp(w_{d'})$
            \begin{enumerate}
                \item If $\widehat{\Var}_2(X_i | X_{S_{d',d''} \setminus \{j\}}) - \widehat{\Var}_2(X_i | X_{S_{d',d''}}) > \nu\widehat{\Var}_2(X_i | X_{S_{d',d''}})$, continue to next iteration of outer loop.
            \end{enumerate}
        \end{enumerate}
        \item Return $\supp(w_{d'})$.
    \end{enumerate}
\end{enumerate}
\vspace{.1cm}
\end{minipage}}
\end{figure}
For the purposes of the analysis, for every pair of sets $S_0 \subset S$ not containing $i$ define (as in Theorem~\ref{thm:noncentral-f-test})
\[ T(S_0,S) := \frac{n - |S|}{|S| - |S_0|}\frac{\|\mathbbm{X}_i - \mathbbm{X} \hat{w}_0\|^2 - \|\mathbbm{X}_i - \mathbbm{X} \hat{w}\|^2}{\|\mathbbm{X}_i - \mathbbm{X} \hat{w}\|^2} = \frac{n - |S|}{|S| - |S_0|} \frac{\|\mathbbm{X} \hat{w} - \mathbbm{X} \hat{w}_0\|^2}{\|\mathbbm X_i - \mathbbm{X} \hat{w}\|^2} \]
where $\hat{w}_0$ is the OLS estimator restricted to $\supp(w_0) \subset S_0$ and $\hat{w}$ is the OLS estimator restricted to $\supp(w) \subset S$.

The following Lemma analyzes the key step in the above algorithm; it shows that when $d'$ equals the true degree of node $i$, the true support is returned. The crucial part which requires that the GGM is attractive is the application of Lemma~\ref{lem:kappa-variance-ferromagnetic}, which guarantees that candidate supports which are far away from the true neighborhood perform much worse than the true neighborhood. This is crucial because there are many candidate neighborhoods far away from the true neighborhood, which means we need an improved bound to handle them and overcome the cost of taking the union bound.
\begin{lemma}\label{lem:ferromagnetic-l0-guarantee}
  In a $\kappa$-nondegenerate attractive GGM, if $i$ is a node of degree $d$ then $\ell_0$ constrained regression over vectors with support size at most $d$ returns the true neighborhood of node $i$ with probability at least $1 - \delta$ as long as $m = \Omega(\log(n)/\kappa^2 + \log(2/\delta)/\kappa^2)$.
\end{lemma}
\begin{proof}
First we consider the randomness over the samples of $X_{\sim i}$, i.e. over $\mathbbm{X}$ with column $i$ removed. By Lemma~\ref{lem:wishart-concentration} and the union bound over all subsets $S$ of $[n] \setminus \{i\}$ with $|S| \le 2d$, it holds with probability at least $1 - \delta/2$ that for all $w$ with $w_i = 0$ and $|\supp(w)| \le 2d$,
\begin{equation}\label{eqn:spectrally-close-small-subsets}
    \frac{1}{2} \E[(w^T X)^2] \le \frac{1}{2} w^T \left(\frac{1}{m} \mathbbm{X}^T \mathbbm{X}\right) w \le \E[(w^T X)^2]
\end{equation}
as long as $m = \Omega(d\log(n) + \log(2/\delta))$. (Recall from Lemma~\ref{lem:d-bounded-by-kappa} that $d \le 1/\kappa^2$, so this holds under the hypothesis of the theorem.) 
We condition on this event and consider the remaining randomness over $\mathbbm{X}_i$.
Let $S^*$ be the set of true neighbors of node $i$ and let $S_0$ be any other subset of size at most $d$. Define $S := S^* \cup S_0$. Since the OLS estimators are defined by projection onto spans of the columns of $\mathbbm{X}$, we can apply the Pythagorean theorem to get
\[  \|\mathbbm{X}_i - \mathbbm{X} \hat{w}_{S^*}\|^2 = \|\mathbbm{X}_i - \mathbbm{X} w_{S}\|^2 + \|\mathbbm{X} \hat{w}_{S^*} - \mathbbm{X} \hat{w}_{S}\|^2 \]
and
\[ \|\mathbbm{X}_i - \mathbbm{X} \hat{w}_{S_0}\|^2 = \|\mathbbm{X}_i - \mathbbm{X} w_{S}\|^2 +\|\mathbbm{X} \hat{w}_{S_0} - \mathbbm{X} \hat{w}_{S}\|^2. \]
Subtracting, we get that
\[  \|\mathbbm{X}_i - \mathbbm{X} \hat{w}_{S_0}\|^2 - \|\mathbbm{X}_{i} - \mathbbm{X} \hat{w}_{S^*}\|^2 =  \|\mathbbm{X} \hat{w}_{S_0} - \mathbbm{X} \hat{w}_{S}\|^2  - \|\mathbbm{X} \hat{w}_{S^*} - \mathbbm{X} \hat{w}_{S}\|^2. \]
To prove the result, it suffices to show with high probability that for any $S_0$ which does not contain $S^*$ that the leftmost term is positive --- then no such $S_0$ can be the minimizer of the $\ell_0$-constrained regression, since $S^*$ corresponds to a feasible point with smaller objective value. We achieve this by showing the right hand side is positive. Observe
\[ \|\mathbbm{X} \hat{w}_{S_0} - \mathbbm{X} \hat{w}_{S}\|^2  - \|\mathbbm{X} \hat{w}_{S^*} - \mathbbm{X} \hat{w}_{S}\|^2 = \frac{d - q}{n - |S|} \|\mathbbm{Y} - \mathbbm{X} \hat{w}_S\|^2 (T(S_0,S) - T(S^*,S)). \]
where $q = |S_0| = |S^*|$
so it suffices to show that $T(S_0,S) - T(S^*,S) \ge 0$. In fact, canceling out denominators, dividing by $\sigma^2$ and rearranging it suffices to show
\[ \frac{1}{\sigma^2} \|\mathbbm{X} \hat{w}_S - \mathbbm{X} \hat{w}_{S_0}\|^2 \ge  \frac{1}{\sigma^2} \|\mathbbm{X} \hat{w}_S - \mathbbm{X} \hat{w}_{S^*}\|^2 \]
where by Theorem~\ref{thm:noncentral-f-test} the left hand side is according to $\chi^2_{d - q}(\gamma)$ with $\gamma := \frac{\min_{\supp(w_0) \subset S} \|\mathbbm{X} (w_0 - w^*)\|^2}{\sigma^2}$ and the right hand side is distributed according to $\chi^2_{d - q}$, where $\sigma^2 := 1/\Theta_{ii}$. Observe by \eqref{eqn:spectrally-close-small-subsets} that
\begin{equation}\label{eqn:l0-gamma-lb}
\gamma \ge \frac{m}{2} \frac{\min_{\supp(w_0) \subset S} \E[(X^T (w_0 - w^*))^2]}{\sigma^2} = \frac{m}{2} \frac{\min_{\supp(w_0) \subset S} \Var(X^T (w_0 - w^*))]}{\sigma^2} \ge \frac{m \kappa^2 (d - q)}{2} \end{equation}
where the last inequality is by Lemma~\ref{lem:kappa-variance-ferromagnetic}, since $w_0$ is supported on $S_0$ which is missing $d - q$ of the neighbors of node $i$. Applying Lemma~\ref{lem:noncentral-chi-squared-concentration}
\[ \Pr(\frac{1}{\sigma^2} \|\mathbbm{X} \hat{w}_S - \mathbbm{X} \hat{w}_{S_0}\|^2 \le (d - q + \gamma) - 2\sqrt{(d - q + 2\gamma)t}) \le e^{-t} \]
and applying Lemma~\ref{lem:chi-squared-concentration}
\[ \Pr(\frac{1}{\sigma^2} \|\mathbbm{X} \hat{w}_S - \mathbbm{X} \hat{w}_{S^*}\|^2  \ge (d - q) + 2\sqrt{(d - q)t} + 2t) \le e^{-t}. \]
Letting $t = \log(4dn^{d - q}/\delta)$, and taking the union bound over the at most $n^{d - q}$ possible values of $S_0$ and then over the at most $d$ possible values of $q$, we find that with probability at least $1 - \delta/2$ for all possible $S_0$ and $q$ that
\[ \frac{1}{\sigma^2} \|\mathbbm{X} \hat{w}_S - \mathbbm{X} \hat{w}_{S_0}\|^2 - \frac{1}{\sigma^2} \|\mathbbm{X} \hat{w}_S - \mathbbm{X} \hat{w}_{S^*}\|^2 \ge \gamma - 2\sqrt{(d - q + 2\gamma)t} - 2\sqrt{(d - q)t} \ge \gamma - 4\sqrt{(d - q + 2\gamma)t}. \]
Finally, we see this is nonnegative as long as $\gamma = \Omega(t) = \Omega((d - q) \log(n) + \log(2/\delta))$, which by \eqref{eqn:l0-gamma-lb} holds as long as $m = \Omega(\frac{\log(n) + \log(2/\delta)}{\kappa^2})$. Therefore the desired result holds with total probability at least $1 - \delta$, completing the proof.
\end{proof}
\begin{theorem}\label{thm:search-and-validate}
  Fix $\delta > 0$.
  In a $\kappa$-nondegenerate attractive GGM, as long as $m = \Omega((1/\kappa^2)\log(n) + \log(2/\delta)/\kappa^2)$ it holds with probability at least $1 - \delta$ that Algorithm~\textsc{SearchAndValidate} with $\nu = \kappa^2/2$ returns the true neighborhood of every node $i$.
\end{theorem}
\begin{proof}
By applying Lemma~\ref{lem:ferromagnetic-l0-guarantee} and taking the union bound over nodes $i$, we know that as long as $m = \Omega((1/\kappa^2)\log(n) + \log(2/\delta)/\kappa^2)$ then with probability at least $1 - \delta/2$ for every node $i$, for $d'$ equal to the true degree of node $i$ that $w_{d'}$ returned in step 2 of Algorithm~\textsc{SearchAndValidate} is supported on exactly the true neighborhood of node $i$. 

Furthermore, conditioned on the previous event (which only involves sample set 1), it holds with probability at least $1 - \delta/2$ by taking the union bound over the possible values of $d',d''$ that (similar to the pruning argument used in analysis of Algorithm~\textsc{GreedyAndPrune}):
\begin{enumerate}
    \item in step 3(a).ii, for every $d'$ less than the true degree of node $i$
    and for $d''$ equal to the true degree of node $i$ that the outer loop continues to the next step by applying Lemma~\ref{lem:estimating-variance-decrement}, Lemma~\ref{lem:t-equals-f}, and Lemma~\ref{lem:kappa-variance} and considering any $j$ in the true neighborhood and missing from the support of $w_{d'}$.
    \item In step 3 when $d'$ equals the true degree of node $i$, step 3(b) is reached and the true support of node $i$ is returned by applying Lemma~\ref{lem:estimating-variance-decrement} and Lemma~\ref{lem:t-equals-f}.
\end{enumerate} 
as long as $m = \Omega((d + 1/\kappa^2)\log(n) + \log(2/\delta)/\kappa^2)$. Using that $d \le 1/\kappa^2$ by Lemma~\ref{lem:d-bounded-by-kappa}, we see the requirement on $m$ holds and as desired, the algorithm succeeds with total probability at least $1 - \delta$.
\end{proof}
A simplified argument in the general (non-attractive) case, using the weaker bound from Lemma~\ref{lem:kappa-variance} instead of Lemma~\ref{lem:kappa-variance-ferromagnetic}, yields the following result in the general case.
\begin{theorem} 
  Fix $\delta > 0$.
  In a $\kappa$-nondegenerate (not necessarily attractive) GGM with maximum degree $d$, as long as $m = \Omega((d/\kappa^2)\log(n) + \log(2/\delta)/\kappa^2)$ it holds with probability at least $1 - \delta$ that Algorithm~\textsc{SearchAndValidate} with $\nu = \kappa^2/2$ returns the true neighborhood of every node $i$.
\end{theorem}
\section{Hybrid $\ell_1$ regression guarantees}
In the next section, we will discuss algorithms for regression and structure learning in general walk-summable models. Since (as we will see) the conditional variance is not supermodular in these models, we need some fundamentally new tools to analyze this setting. It turns out that we will need to analyze a variant of $\ell_1$-constrained least squares   regression, which we do in this section as preparation.
\begin{defn}
We define the \emph{hybrid $\ell_1$-regression model} to be given by
\[ Y = \langle w^*, X - \E[X | Z] \rangle + a^* Z + \xi \]
where $\|w\|_1 \le W$ and conditioned on $Z$, $X - \E[X | Z] \sim N(0,\Sigma)$ with $\Sigma : n \times n$, $\Sigma_{ii} \le R^2$ for all $i$, $\E Z^2 = 1$ (w.l.o.g.), and $\E \xi^2 = \sigma^2$ with the noise $\xi$ independent of $X,Z$.
\end{defn}
The corresponding function class is
\[ \mathcal{F} := \{(x,z) \mapsto \langle w, x - \E[X | Z = z] \rangle + a z : \|w\|_1 \le W\} = \{(x,z) \mapsto \langle w, x \rangle + a' z : \|w\|_1 \le W \}.  \]
and the \emph{Empirical Risk Minimizer} (ERM) is given by taking the minimizer of
\[ \min_{\|w\|_1 \le W, a'} \hat{\E}[(Y - \langle w, X \rangle - a' Z)^2]. \]
As mentioned in the introduction, it will be crucial in the analysis to look at the parameterization with $a$ instead of $a'$ even though algorithmically the ERM will be computed using the variable $a'$ (as the change of basis given by subtracting off the conditional expectations is unknown and could only be approximated from data).
\subsection{Guarantees for Empirical Risk Minimization (ERM)}
There is a vast literature on generalization bounds for empirical risk minimization (and natural variants) using tools such as (local) Rademacher complexity, stability, etc. (see e.g. \cite{bartlett2005local,srebro2010smoothness,shalev2014understanding} and many related references); however, many of these methods are not well-optimized for our setting because the noise and covariates are drawn from unbounded distributions and the squared loss is not uniformly Lipschitz (see the discussion in \cite{mendelson2014learning}). Fortunately, the framework developed in \cite{mendelson2014learning} avoids these issues and we are able to use it directly to give a good bound on the excess risk of the empirical risk minimizer. 
\subsubsection{Background: Learning without Concentration Framework}
We recall the main result of \cite{mendelson2014learning}. Let $\mathcal{F}$ be a class of (measurable) functions. Let $X,Y$ be arbitrary random variables, suppose that $f^*$ is a minimizer of $\E[(Y - f(X))^2]$ over $f \in \mathcal{F}$ (which we assume exists) and define $\xi := Y - f^*(X)$. Let $\|f\|_{L_2} = \sqrt{\E[f^2]}$ and let $D_2(f)$ be the $L_2$ ball of radius $1$ around $f$, i.e. $D_2(f) = \{g : \E[(g - f)^2] = 1 \}$.
The following two quantities, defined by fixed point equations, appear in the generalization bound: the intrinsic parameter (which does not depend on the noise model)
\[ \beta^*_m(\gamma) = \inf \left\{r > 0 : \E \sup_{f \in \mathcal{F} \cap r D_{f^*}}\left|\frac{1}{\sqrt{m}} \sum_{i = 1}^m \epsilon_i (f - f^*)(X_i)\right| \le \gamma r \sqrt{m}\right\} \]
and the noise-sensitive parameter
\[ \alpha^*_m(\gamma,\delta) = \inf \left\{s > 0 : \Pr\left(\sup_{f \in \mathcal{F} \cap s D_{f^*}} \left|\frac{1}{\sqrt{m}} \sum_{i = 1}^m \epsilon_i \xi_i (f - f^*)(X_i)\right| \le \gamma s^2 \sqrt{m}\right) \ge 1 - \delta \right\}.\]
\begin{theorem}[Theorem 3.1, \cite{mendelson2014learning}]\label{thm:learning-without-concentration}
  Suppose $\mathcal{F}$ is a closed, convex class of functions and $f^*,X,Y,\alpha^*,\beta^*$ are defined as above. Let $\tau > 0$, define
  \[ q := \inf_{f \in \mathcal{F} - \mathcal{F}} \Pr(|f| \ge 2 \tau \|f\|_{L_2}) \]
  and assume that $q > 0$ (this is called the \emph{small-ball condition}).
  Then for any $\gamma < \tau^2 q/16$ and for every $\delta > 0$ it holds that
  for any $\hat{f}$ which is an empirical risk minimizer for i.i.d. samples $\{(X^{(i)}, Y^{(i)})\}_{i = 1}^m$,
  \[ \|\hat{f} - f^*\|_{L_2} \le 2 \max\left\{ \alpha^*_m(\gamma,\delta/4), \beta^*_m(\tau q/16) \right\} \]
  with probability at least $1 - \delta - e^{-mq/2}$.
\end{theorem}
\subsubsection{ERM Risk Bound}
We return to the specific setting of hybrid $\ell_1$-constrained regression and prove our desired bound. 
\begin{theorem}\label{thm:hybrid-erm-bound}
  As long as $m = \Omega(\log(n/\delta))$, if $\hat{w},\hat{a}'$ is the empirical risk minimizer for hybrid L1 regression from $m$ i.i.d. samples then
  \[ \E[(\E[Y | X,Z] - \langle \hat{w}, X \rangle - \hat{a}'Z)^2] = O\left(RW \sigma \sqrt{\frac{ \log(2n/\delta)}{m}} +  \frac{\sigma^2 \log(4/\delta)}{m} + \frac{R^2 W^2 \log(n)}{m} \right) \]
  with probability at least $1 - \delta$.
\end{theorem}
\begin{proof}
We first deal with the small-ball condition. Let $\tau = 1/2$.
Observe that for any $f_1,f_2 \in \mathcal{F}$ that $f_1(X,Z) - f_2(X,Z)$ has a univariate Gaussian distribution, therefore
\[ q := \Pr(|f| \ge 2 \tau \|f\|_{L_2}) = 1 - \frac{1}{\sqrt{2\pi}} \int_{-2\tau}^{2\tau} e^{-x^2/2} dx \ge 1/4. \]
We take $\gamma = 1/300 < \tau^2 q/32$.  

We now bound $\beta^*$. We have
\begin{align*} 
\E &\sup_{f \in \mathcal{F} \cap r D_{f^*}}\left|\frac{1}{\sqrt{m}} \sum_{i = 1}^m \epsilon_i (f - f^*)(X_i)\right| \\
&= \E \sup_{f \in \mathcal{F} \cap r D_{f^*}}\left|\frac{1}{\sqrt{m}} \sum_{i = 1}^m \epsilon_i (\langle w - w^*, X_i - \E[X_i | Z_i] \rangle + (a - a^*)Z) \right| \\
&\le 2RW  \E\left\|\frac{1}{\sqrt{m}} \sum_{i = 1}^n \epsilon_i \frac{X_i - \E[X_i | Z_i]}{W} \right\|_{\infty}  +  \sup_{f \in \mathcal{F} \cap r D_{f^*}} |a - a^*| \E |Z| \\
&\le C (RW \sqrt{\log(n)} + \sup_{f \in \mathcal{F} \cap r D_{f^*}} |a - a^*|)
\end{align*}
where the first inequality is by Holder's inequality and the triangle inequality, and the second is by the standard Gaussian tail bound combined with the union bound.
To complete the bound observe that 
\[ \E[(\langle w - w^*, X - \E[X | Z] \rangle + (a - a^*) Z)^2] \ge (a - a^*)^2 \]
so $a - a^* \le r$ and 
\[ \E \sup_{f \in \mathcal{F} \cap r D_{f^*}}\left|\frac{1}{\sqrt{n}} \sum_{i = 1}^m \epsilon_i (f - f^*)(X_i)\right| \le C(RW \sqrt{\log(n)} + r). \]
This is smaller than $\gamma r \sqrt{m}$ as long as $r = \Omega(\frac{RW}{\gamma} \sqrt{\frac{\log n}{m}})$ so $\beta^*_m = O(\frac{RW}{\gamma} \sqrt{\frac{\log n}{m}})$. 

We proceed to bound $\alpha^*$ similarly.
\begin{align*}
\sup_{f \in \mathcal{F} \cap s D_{f^*}}&\left|\frac{1}{\sqrt{m}} \sum_{i = 1}^m \epsilon_i \xi_i (f - f^*)(X_i)\right|\\
&= \sup_{f \in \mathcal{F} \cap s D_{f^*}}\left|\frac{1}{\sqrt{m}} \sum_{i = 1}^m \epsilon_i \xi_i (\langle w - w^*, X_i - \E[X_i | Z_i] \rangle + (a - a^*)Z)\right| \\
&\le C (RW \sigma \sqrt{\log(2n/\delta)} + \sigma s\sqrt{\log(4/\delta)})
\end{align*}
with probability at least $1 - \delta$ as long as $m \ge m_1 = O(\log(n/\delta))$, where the inequality is by Holder's inequality and $|a - a^*| \le s$ (as before), Bernstein's inequality (Theorem 2.8.2 of \cite{vershynin2018high}) using that the product of sub-Gaussian r.v. ($\xi_i$ and $X_i - \E[X_i | Z_i]$) is sub-exponential (Lemma 2.7.7 of \cite{vershynin2018high}), and the union bound. The last quantity is upper bounded by $\gamma s^2 \sqrt{m}$ as long as $s^2 = \Omega(\frac{\sigma}{\gamma} \sqrt{\frac{\log(2n/\delta)}{m}})$ and $s = \Omega(\frac{\sigma}{\gamma} \sqrt{\frac{\log(4/\delta)}{m}})$. Therefore
\[ (\alpha^*)^2 = O\left(\frac{RW \sigma}{\gamma} \sqrt{\frac{\log(2n/\delta)}{m}} +  \frac{\sigma^2 \log(4/\delta)}{\gamma^2 m}\right). \]
Combining our estimates, it follows from Theorem~\ref{thm:learning-without-concentration} that
\[ \E[(\hat{f} - f^*)^2] = O((\alpha^*_m)^2 + (\beta^*_m)^2) = O\left(\frac{RW \sigma}{\gamma} \sqrt{\frac{ \log(2n/\delta)}{m}} +  \frac{\sigma^2 \log(4/\delta)}{\gamma^2 m} + \frac{R^2 W^2 \log(n)}{\gamma m} \right)\]
with probability at least $1 - \delta - e^{-m/8} \ge 1 - 2\delta$ as long as $m = \Omega(\log(1/\delta) + m_1) = \Omega(\log(d/\delta))$. Since $\gamma$ is just a constant, this gives the result.
\end{proof}
\subsection{Guarantees for Greedy Methods}\label{sec:greedy-l1}
In this section we show that a simple greedy method can also solve this high-dimensional regression
problem with the correct dependence on $n$, albeit with slightly worse dependence on the
other parameters. This is conceptually important as it shows that examples breaking greedy algorithms (in the sense of requiring $\omega(\log(n))$ sample complexity)
 also suffice to break analyses based on bounded $\ell_1$-norm.
\begin{lemma}\label{lem:hybrid-greedy-step}
In the hybrid $\ell_1$-regression model, there exists an input
coordinate $j$ such that
\[ \Var(\E[Y | X,Z] \mid Z, X_j) \le \Var(\E[Y | X,Z] \mid Z)\left(1 - \frac{\Var(\E[Y | X,Z] \mid Z)}{R^2W^2}\right). \]
\end{lemma}
\begin{proof}
By expanding, applying Holder's inequality and using the assumption on $R$ we have
\begin{align*} 
\Var(\E[Y | X,Z] \mid Z) 
&= \sum_{j} w_j \Cov(\E[Y | X,Z], X_j \mid Z) \\
&\le W \max_j |\Cov(\E[Y|X,Z], X_j \mid Z)| \\
&\le RW \max_j \left|\Cov\left(\E[Y | X,Z], \frac{X_j}{\sqrt{\Var(X_j | Z)}} \ \Big|\  Z\right)\right|.
\end{align*}
Let $j$ be the maximizer. Then by Lemma~\ref{lem:gaussian-conditional-expectation},
\[ \Var(\E[Y | X,Z] \mid Z) - \Var(\E[Y | X,Z] \mid Z, X_j) = \frac{\Cov(\E[Y | X,Z], X_j \mid Z)^2}{\Var(X_j \mid Z)} \ge \frac{\Var(\E[Y | X,Z] \mid Z)^2}{R^2W^2}.  \]
Rearranging gives that
\[ \Var(\E[Y | X,Z] \mid Z, X_j) \le \Var(\E[Y | X,Z] \mid Z)\left(1 - \frac{\Var(\E[Y | X,Z] \mid Z)}{R^2W^2}\right). \]
\end{proof}
The above bound naturally leads to analyzing the recursion $x \mapsto x - cx^2$, which we do in the next Lemma. 
\begin{lemma}\label{lem:recurrence-bd}
Suppose that $x_1 \le 1/2c$ and $x_{t + 1} \le (1 - cx_t) x_t$ for some $c < 1$. Then
\[ x_t \le \frac{1}{c(t + 1)} \]
\end{lemma}
\begin{proof}
We prove this by induction. Observe that $x(1 - cx)$ is an increasing function in $x$ for $x \le \frac{1}{2c}$ since $1/2c$ corresponds to the vertex of the parabola, so using the assumption and the induction hypothesis,
\[ x_{t} \le x_{t - 1}(1 - cx_{t - 1}) \le 1/ct - 1/ct^2 = \frac{t - 1}{ct^2} \le \frac{t - 1}{c(t^2 - 1)} \le \frac{1}{c(t + 1)}.  \]
\end{proof}
\begin{lemma}\label{lem:hybrid-greedy-init}
In the hybrid $\ell_1$-regression model,
\[ \Var(\E[Y | X,Z] \mid Z) \le R^2 W^2.\]
\end{lemma}
\begin{proof}
By expanding, using Holder's inequality and Cauchy-Schwartz
\begin{align*}
\Var(\E[Y | X,Z] \mid Z) 
&= \sum_j w_J \Cov(\E[Y | X,Z], X_j \mid Z) \\
&\le W \max_j |\Cov(\E[Y | X,Z], X_j \mid Z) \\
&\le W \max_j \sqrt{\Var(\E[Y | X,Z] \mid Z) \Var(X_j \mid Z)} \le RW \sqrt{\Var(\E[Y | X,Z] \mid Z)}
\end{align*}
so $\Var(\E[Y | X,Z] \mid Z) \le R^2 W^2$.
\end{proof}
\begin{remark}[Connection to Approximate Caratheodory]\label{rmk:approximate-caratheodory}
From the previous two lemmas, we can give a ``matching pursuit'' proof of the approximate Caratheodory theorem, which says that vectors of bounded $\ell_1$-norm are well approximated in $\ell_2$ by sparse vectors \cite{vershynin2018high}. The standard proof of this result is probabilistic. Another proof, in a similar spirit, is given by using the guarantees of the Frank-Wolfe algorithm (see \cite{bubeck2015convex}).
\end{remark}
The remaining task is to analyze the behavior of the iteration under noise, which gives the main result:
\begin{theorem}\label{thm:omp-hybrid-regression}
  For any $\epsilon \in (0,1)$, iterate $t$ of OMP in the hybrid regression model satisfies
  \[ \Var(\E[Y | X,Z] | Z, X_{S_t}) \le \epsilon \sigma^2 \]
  as long as $t = \Omega(R^2W^2/\epsilon \sigma^2)$ and $m = \Omega(\frac{R^2W^2}{\epsilon^2 \sigma^2}(t\log(4n) + \log(4/\delta)))$. 
\end{theorem}
\begin{proof}
The argument is structured similarly to the proof of Lemma~\ref{lem:ferromagnetic-subset-selection}.
Fix $\epsilon \in (0,1)$ to be optimized later: we bound the number of steps of OMP during which $\Var(\E[Y | X,Z] | Z, X_{S_t}) \ge \epsilon \sigma^2$. Note that once this bounds holds for some $t$, it holds for all larger $t$ by the law of total variance. Fix an integer $T > 0$ to be optimized later.

First observe from Lemma~\ref{lem:hybrid-greedy-step} (applied after conditioning out $X_{S_t}$) that there exists a node $j^*$ such that
\[ \Var(\E[Y | X,Z] | Z, X_{j^*}, X_{S_t}) \le \Var(\E[Y | X,Z] | Z, X_{S_t})(1 - \frac{\Var(\E[Y | X,Z] | Z, X_{S_t})}{R^2 W^2}). \]
From Lemma~\ref{lem:estimating-variance-decrement-v} and taking the union bound over all sets $S$ of size $|S| \le T$ we have
\[ \left|\sqrt{\frac{1}{\Var(Y | X_{S \setminus j}) - \sigma^2}} \frac{|\hat{w}_j|}{\sqrt{(\hat \Sigma^{-1})_{jj}}} - \sqrt{\gamma'}\right| \le \sqrt{\frac{\Var(Y | X_S)}{\Var(Y | X_{S \setminus j}) - \sigma^2} \cdot \frac{2 \log(n^T/\delta)}{m}} + \sqrt{\frac{\gamma'}{64}} \le  \sqrt{\frac{1 + \epsilon}{\epsilon} \cdot \frac{2 \log(n^T/\delta)}{m}} + \sqrt{\frac{\gamma'}{64}}\]
using that $(1 + x)/x = 1/x + 1$ is monotone decreasing, where
\[ \gamma' = \gamma'(S,j) := \frac{\Var(X_i | Z,X_{S \setminus \{j\}}) - \Var(X_i | Z,X_S)}{\Var(X_i | Z,X_{S \setminus \{j\}}) - \sigma^2} . \]
Note that $\gamma'(S,j^*) \ge \epsilon \sigma^2/R^2 W^2$. 
Therefore as long as $m = \Omega(\frac{R^2 W^2}{\epsilon^2 \sigma^2}(T\log(4n) + \log(4/\delta)))$ then OMP chooses a node $j$ s.t. 
\[ \Var(\E[Y | X,Z] | Z, X_{j}, X_{S_t}) \le \Var(\E[Y | X,Z] | Z, X_{S_t})(1 - \frac{\Var(\E[Y | X,Z] | Z, X_{S_t})}{2 R^2 W^2})\]
as long as $|S_t| \le T$. Applying Lemma~\ref{lem:hybrid-greedy-init} and Lemma~\ref{lem:recurrence-bd} we find that
\[ \Var(\E[Y | X,Z] | Z, X_{S_t}) \le \frac{2R^2W^2}{t + 1} \]
for $t \le T$. Therefore if $T \ge t \ge 2R^2W^2/\epsilon \sigma^2$ we are guaranteed that
$\Var(\E[Y | X,Z] | Z, X_{S_t}) \le \epsilon \sigma^2$. Taking $\epsilon = 2 R^2W^2/T\sigma^2$ gives the result.
\end{proof}
\section{Regression and Structure Learning in Walk-Summable Models}

\subsection{Failure of (weak) supermodularity in SDD models}\label{sec:sdd-examples}
The following example shows that the conditional variance is not supermodular in the SDD case, unlike in the attractive/ferromagnetic case. 
\begin{example}\label{example:no-submodularity}
Consider the GGM given by SDD precision matrix
\[ \Theta = \begin{bmatrix} 
1 & -1/2 & -1/2 \\
-1/2 & 1 & 1/2 \\
-1/2 & 1/2 & 1
\end{bmatrix} \]
and label the nodes (in order) by $i,j,k$. One can see (e.g. by computing effective resistances in the lifted graph) that $2\Var(X_i) = 3$, that $2\Var(X_i | X_j) = 2\Var(X_i | X_k) = 8/3$, and $2\Var(X_i | X_j,X_k) = 2$. Since $3 - 8/3 = 1/3 < 2/3 = 8/3 - 2$ this violates supermodularity.
\end{example}
The above example alone does not rule out the possibility that (negative) conditional variances in SDD models always have \emph{submodularity ratio} introduced by \cite{das2011submodular} lower bounded by a constant. We recall the definition next:
\begin{defn}[\cite{das2011submodular}]
The \emph{submodularity ratio}  $\gamma(k)$ of a function on subsets of a universe $U$, $f : 2^U \to \mathbb{R}_{\ge 0}$ is defined to be
\[ \gamma(k) := \min_{L \subset U, |S| \le k, L \cap S = \emptyset} \frac{\sum_{x \in S} f(L \cup \{x\}) - f(L)}{f(L \cup S) - f(L)}\]
Note that $\gamma(k) \ge 1$ for a  submodular function.
\end{defn}
 The significance of this ratio for a function $f$ is that if the ratio is lower bounded by a constant then similar guarantees for submodular maximization follow (\cite{das2011submodular}); for this reason such an $f$ is sometimes called \emph{weakly submodular} (as in e.g. \cite{elenberg2016restricted}). Now, we give a counterexample showing that for general SDD matrices, this ratio can be arbitrarily small.
\begin{example}\label{example:no-apx-submodularity}
Fix $M > 0$ large.
Let $\epsilon > 0$ be a parameter to be taken small, and consider the following precision matrix, which is SDD as long as $\epsilon < 1/2 < M$:
\[ \Theta = 
\begin{bmatrix}
1 & -\epsilon & \epsilon \\
-\epsilon & M & \epsilon - M \\
\epsilon & \epsilon - M & M
\end{bmatrix}. \]
This has inverse
\[ \Theta^{-1} =
\begin{bmatrix}
(\epsilon-2 M)/(\epsilon+2 \epsilon^2-2 M) & -(\epsilon/(\epsilon+2 \epsilon^2-2 M)) & \epsilon/(\epsilon+2 \epsilon^2-2 M) \\
-(\epsilon/(\epsilon+2 \epsilon^2-2 M)) & (\epsilon^2-M)/(\epsilon^2+2 \epsilon^3-2 \epsilon M) & (\epsilon+\epsilon^2-M)/(\epsilon^2+2 \epsilon^3-2 \epsilon M) \\
\epsilon/(\epsilon+2 \epsilon^2-2 M) & (\epsilon+\epsilon^2-M)/(\epsilon^2+2 \epsilon^3-2 \epsilon M) & (\epsilon^2-M)/(\epsilon^2+2 \epsilon^3-2 \epsilon M) \\
\end{bmatrix} \]
so
\[ \Var(X_1) - \frac{1}{\Theta_{11}} = \frac{-2\epsilon^2}{\epsilon + 2\epsilon^2 - 2M} \]
and (by computing the inverse of the top-left 2x2 submatrix of $\Theta$) we find
\[ \Var(X_1 | X_3) - \frac{1}{\Theta_{11}} = \frac{M}{M - \epsilon^2} - 1 = \frac{\epsilon^2}{M - \epsilon^2} \]
and the difference is
\[ \Var(X_1) - \Var(X_3) = \frac{\epsilon^3}{(M - \epsilon^2)(2M - 2\epsilon^2 - \epsilon)} \]
Therefore the \emph{submodularity ratio} $\gamma = \gamma(2)$ for $f(S) = \Var(X_1) - \Var(X_1 | X_S)$ is upper bounded by (taking $L = \emptyset$)
\[ \gamma \le \frac{f(\{2\}) + f(\{3\})}{f(\{2,3\})} = \Theta\left(\frac{\epsilon^3/M^2}{\epsilon^2/M}\right) = \Theta(\epsilon/M) \]
which is clearly arbitrarily small. 
\end{example}
\begin{remark}[Submodularity ratio and $\kappa$]\label{rmk:submodularity-ratio}
It's possible to show, based on Lemma~\ref{lem:hybrid-greedy-step} and the bounds in the proof of Theorem~\ref{thm:ws-regression}
to derive a partial lower bound for the submodularity ratio when
we consider $S \subset T$ and restrict to $j$ which are neighbors of $i$, by showing:
\[ f(S \cup \{j\}) - f(S) \ge \frac{\kappa^2}{4d}(f(U) - f(S)) \ge \frac{\kappa^2}{4d}(f(T \cup \{j\}) - f(T)) \]
using the monotonicity of $f$ (which follows from the law of total variance) in the last step, and under the assumption that the model is $\kappa$-nondegenerate and $d$-sparse. The above example shows that this dependence on $\kappa$ is tight: by taking a fixed small $\epsilon$ and sending $M \to \infty$, the submodularity ratio can be as small as $O(\kappa^2)$ since $\kappa = \epsilon/\sqrt{M}$ in this model. It remains unclear if the submodularity ratio can be lower bounded in general in $\kappa$-nondegenerate models; even if such a bound did hold it could not be used to prove Theorem~\ref{thm:ws-regression} since that result holds without a $\kappa$-nondegeneracy assumption.
\end{remark}
\subsection{Sparse regression}
In this section we describe an algorithm to find a good predictor of node $X_i$ with bounded degree $d$ in a walk-summable GGM. To simplify the analysis, we assume the data has been split into 3 equally sized sample sets, each of size $m$; when there is no explicit mention, averages are taken over sample set 1.

The algorithm is conceptually straightforward: it does a single greedy step and then sets up an $\ell_1$-constrained regression. The only complication is that we do not know $1/\Theta_{ii}$ a priori, but this appears in the $\ell_1$-norm of the obvious regression we want to setup. Since we have multiplicative estimates for $1/\Theta_{ii}$, we can deal with this by searching over the possible values on a log scale.

\begin{figure}
\fbox{\begin{minipage}{\textwidth}
\vspace{.1cm}
Algorithm \textsc{WS-Regression($\gamma,d$)}: 
\begin{enumerate}
    \item Choose $j$ to minimize $\widehat{\Var}(X_i | X_j)$.
    \item Let $s_0^2 := \exp(\lfloor \log(\widehat{\Var}(X_i | X_j)/8d) \rfloor - 1)$.
    \item For $\ell$ in $0$ to $\lceil\log(8d) + 3\rceil$:
    \begin{enumerate}
        \item Let $s_{\ell}^2 := s_0 e^{\ell}$
        \item Solve for $w, a$ in 
    \[ \min_{w, a : \|w\|_1 \le \lambda} \hat{\E}_2 \left[\left(X_i - \sum_{k \notin \{i,j\}} w_k \frac{X_{k}}{\sqrt{\widehat{\Var}(X_k | X_j)}} - a X_j\right)^2\right]\]
    where $\lambda = \sqrt{2d} s_{\ell}$ and  $\hat{\E}_2$ is empirical expectation over sample set 2.
        \item Let $\hat{\sigma}^2 :=  \hat{\E}_3\left[\left(X_i - \sum_{k \notin \{i,j\}} w_k \frac{X_{k}}{\sqrt{\widehat{\Var}(X_k | X_j)}} - a X_j\right)^2\right]$ where $\hat{\E}_3$ is empirical expectation over sample set 3.
        If  $\lambda^2 \ge 2d\gamma^2 \hat{\sigma}^2$ (equivalently, $s_{\ell}^2 \ge \gamma^2  \hat{\sigma}^2$), 
        then exit the loop.
    \end{enumerate}
    \item Return $w,a,j,\hat{\sigma}^2$.
\end{enumerate}
\end{minipage}}
\vspace{.1cm}
\end{figure}
We show this algorithm gives a result for sparse linear regression under the walk-summability assumption which (1) depends on sparsity only, not on norms (unlike the slow rate bound for LASSO) and (2) is computationally efficient (unlike brute force $\ell_0$-constrained regression).
\begin{theorem}\label{thm:ws-regression}
  Let $i$ be a node of degree $d$ in an SDD GGM and $\sigma^2 := 1/\Theta_{ii}$. Then WS-Regression($\gamma$) with $\gamma^2 = 2$ returns $w,a$ such that 
  \[ \E\left[\left(\E[X_i | X_{\sim i}] - \sum_{k \notin \{i,j\}} w_k \frac{X_k}{\sqrt{\widehat{\Var}(X_k | X_j)}}  - a X_j\right)^2\right] =  O\left(\sigma^2 \sqrt{\frac{d\log(2n/\delta)}{m}}\right)\]
  and $\hat{\sigma}^2$ s.t. $1/2 \le \Theta_{ii} \hat{\sigma}^2 \le 2$
  with probability at least $1 - \delta$, as long as $m = \Omega(\log(n/\delta))$.
\end{theorem}
\begin{proof}
By Lemma~\ref{lem:bound-after-conditioning-ij}, for any $k \sim i$ we have
$\Var(X_i | X_j) \le 1/|\Theta_{ik}|$
therefore if we take $j^*$ which minimizes $\Var(X_i | X_{j^*})$ then
\[ \Var(X_i | X_{j^*}) \le 1/|\Theta_{ij}| \]
for all $j$. Similarly, applying Lemma~\ref{lem:bound-after-conditioning} we know that
\[ \Var(X_i | X_{j^*}) \le \frac{4d}{\Theta_{ii}} \]
By using Lemma~\ref{lem:ols-sigmahat} and taking the union bound over the randomness of sample set 1, we may assume that for every $j,k$, $\Var(X_k | X_j)/\sqrt{2} \le \widehat{\Var}(X_k | X_{j}) \le \sqrt{2}\Var(X_k | X_j)$,
with probability at least $1 - \delta/3$ as long as $m = \Omega(\log(n/\delta))$. We condition on this event.
Then for the $j$ chosen in step 1 of the algorithm, we have that
\[ \Var(X_i | X_j) \le \sqrt{2} \widehat{\Var}(X_i | X_j) \le \sqrt{2} \widehat{\Var}(X_i | X_{j^*}) \le 2 \Var(X_i | X_{j^*}) \le 2/|\Theta_{ik}| \]
for all $i \sim k$, and similarly
\begin{equation} \label{eqn:theta_ii-mult-estimate}
\Var(X_i | X_j) \le \frac{8d}{\Theta_{ii}}.
\end{equation}
Furthermore,
\[ \Var\left(\frac{X_k}{\sqrt{\widehat{\Var}(X_k | X_j)}} \middle| X_j\right) \le \sqrt{2} \]
and
\begin{align*}
 \sum_k \frac{|\Theta_{ik}|}{\Theta_{ii}} \sqrt{\widehat{\Var}(X_k | X_j)} 
 &\le  \sum_k \frac{|\Theta_{ik}|}{\Theta_{ii}} \sqrt{2\Var(X_k | X_j)} \\
 &\le \sum_k \frac{|\Theta_{ik}|}{\Theta_{ii}} \sqrt{2(1/|\Theta_{ik}| + \Var(X_i | X_j))} \\
 &\le \sum_k \frac{|\Theta_{ik}|}{\Theta_{ii}} \sqrt{2(3/|\Theta_{ik}|)} 
 = \frac{\sqrt{6}}{\Theta_{ii}} \sum_k \sqrt{|\Theta_{ik}|}
 \le \sqrt{6d/\Theta_{ii}}
\end{align*}   
using Lemma~\ref{lem:sdd-smooth-variance} in the second inequality and Cauchy-Schwartz
and the SDD property in the final inequality. Given \eqref{eqn:theta_ii-mult-estimate} we know that for one of the values of $\ell$ satisfies $e/\Theta_{ii} \le s_{\ell}^2 \le e^2/\Theta_{ii}$; call this $\ell^*$. By Theorem~\ref{thm:hybrid-erm-bound} we have that with probability at least $1- \delta/3$ that for all of the loop iterations where $1/\Theta_{ii} \le s_{\ell}^2$ (so the global optimal $w^*,a$ is in the constraint set) and $\ell \le \ell^*$
\begin{equation}\label{eqn:ell-star}
\E\left[\left(X_i - \sum_{k \notin \{i,j\}} w_k \frac{X_k}{\sqrt{\Var(X_k | X_j)}} - a X_j\right)^2\right] = O\left(\sqrt{1/\Theta_{ii}} \sqrt{24d/\Theta_{ii}} \sqrt{2} \sqrt{\frac{\log(n^2/\delta)}{m}}\right) 
\end{equation}
as long as $m = \Omega(\log(n/\delta))$, using that $d \le n$ in the union bound.
Condition on this and consider only the randomness over sample set 3. By Bernstein's inequality and the union bound over the loop iterations, with probability at least $1 - \delta/3$ as long as $m = \Omega(\log(n/\delta))$, for the above value of $\ell  = \ell^*$ we have that the test in 3(c) succeeds and the loop exits, and that if the loop exited in a previous iteration then $\frac{1}{\Theta_{ii}} = \Var(X_i | X_{\sim i}) \le s_{\ell}^2$ so we can apply the above guarantee \eqref{eqn:ell-star}, giving the result.
\end{proof}
\subsection{Structure learning}
\begin{figure}
\fbox{\begin{minipage}{\textwidth}
\vspace{.1cm}
Algorithm \textsc{HybridMB($\tau,\gamma,d$)}:
\begin{enumerate}
    \item We suppose the samples are split into $3$ equally sized sets as in the description of $\textsc{WS-Regression}$.
    \item For every node $i$, apply $\textsc{WS-Regression}$ which returns $w(i),a(i),j(i),\hat{\sigma}^2(i)$.
    \item Define $u(i)_{j(i)} = a(i)$ and $u(i)_{k} = \frac{w(i)_k}{\sqrt{\widehat{\Var}(X_k | X_j)}}$.
    \item Let $E = \{\}$.
    \item For every pair of nodes $a,b$:
    \begin{enumerate}
        \item If $u(a)_b^2 \hat{\sigma}^2(b) \ge \tau \hat{\sigma}^2(a)$ and $u(b)_a^2 \hat{\sigma}^2(a) \ge \tau \hat{\sigma}^2(b)$: add $(i,j)$ to $E$.
    \end{enumerate}
    \item Return edge set $E$.
\end{enumerate}
\end{minipage}}
\vspace{.1cm}
\end{figure}
\begin{theorem}\label{thm:structure-learning-via-hybrid}
Fix an SDD, $\kappa$-nondegenerate GGM.
Algorithm \textsc{HybridMB} with $\tau = \kappa^2/8,\gamma=2$ returns the true neighborhood of every node $i$ with probability at least $1 - \delta$ as long as $m \ge m'_1$, where $m'_1 = O((d/\kappa^4) \log(n/\delta))$ where $d$ is the max degree in the graph.
\end{theorem}
\begin{proof}
By Theorem~\ref{thm:ws-regression} and the union bound, we may assume with probability at least $1 - \delta$, as long as $m = \Omega((d/\kappa^4)\log(n/\delta))$ that for every node $i$ we have $u(i)$ such that 
\[ \E\left[\left(\E[X_i | X_{\sim i}] - \sum_{k \ne i} u(k) X_k \right)^2\right] \le \frac{\kappa^2}{16 \Theta_{ii}} \]
and $\hat{\sigma}^2(i)$ which is within a factor of $2$ of $1/\Theta_{ii}$. Applying the law of total variance and \eqref{eqn:ggm-conditional} we find that
\[ \left(\frac{u(k)}{\sqrt{\Theta_{kk}}} + \frac{\Theta_{ik}}{\Theta_{ii} \sqrt{\Theta_{kk}}}\right)^2 = \left(u(k) + \frac{\Theta_{ik}}{\Theta_{ii}}\right)^2 \Var(X_k | X_{\sim k}) \le \frac{\kappa^2}{64\Theta_{ii}} \]
so if $i$ and $k$ are not neighbors, then $\Theta_{ik} = 0$ so
\[ u(k)^2\hat{\sigma}^2(k) \le 2 u(k)^2/\Theta_{kk} \le \frac{\kappa^2 \hat{\sigma}_i^2}{16} \]
and if they are then $|\Theta_{ik}| \ge \kappa\sqrt{\Theta_{ii}\Theta_{kk}}$ so using the reverse triangle inequality 
\[ u(k)^2\hat{\sigma}^2(k) \ge (1/2)u(k)^2/\Theta_{kk} \ge (1/2)(\kappa^2/\sqrt{\Theta_{ii}} - \kappa/8\sqrt{\Theta_{ii}}) \ge (7/16) \kappa^2/\sqrt{\Theta_{ii}} \ge (7/32) \kappa^2 \hat{\sigma}^2(i). \]
From these inequalities we see that in step 5 (a) exactly the correct edges are chosen.
\end{proof}
\begin{theorem}\label{thm:greedy-and-prune-sdd}
Fix an SDD, $\kappa$-nondegenerate GGM.
Algorithm \textsc{GreedyAndPrune} with $\tau = \kappa^2/8$ and $T = \Theta(d/\kappa^2)$ returns the true neighborhood of every node $i$ with probability at least $1 - \delta$ as long as $m = \Omega((d^2/\kappa^6) \log(n/\delta))$ where $d$ is the max degree in the graph.
\end{theorem}
\begin{proof}
The proof is the same as for Theorem~\ref{thm:structure-learning-via-hybrid} except that we use Theorem~\ref{thm:omp-hybrid-regression} instead of Theorem~\ref{thm:hybrid-erm-bound}, and use the 
slightly different pruning analysis from the proof of Theorem~\ref{thm:greedy-ferromagnetic}.
\end{proof}
\begin{remark}[Implementation]
In experiments, to reduce the number of free parameters in \textsc{HybridMB} we define $\gamma' = 2d \gamma^2$ and note that using $\gamma'$ instead of $\gamma$ actually allows $d$ to be eliminated as a parameter. We also use a single sample set instead of sample splitting; we expect that the algorithm can still be proved correct without the splitting, at the cost of a more lengthy analysis.
\end{remark}
\begin{remark}[Guarantees under $\ell_1$-bounded assumption]\label{rmk:l1-bounded}
For completeness, we state results for our algorithms under the $\ell_1$-bounded assumption used in previous work like \cite{cai2011constrained,cai2016estimating}. This is straightforward, as we can ignore
the analysis of the first step and simply use the a priori estimate for the $\ell_1$ norm, which only shrinks under conditioning. Following the proofs of Theorem~\ref{thm:structure-learning-via-hybrid} and Theorem~\ref{thm:greedy-and-prune-sdd} give that \textsc{HybridMB} achieves a sample complexity of
$O(\frac{M^2 \log(n/\delta)}{\kappa^4})$ for structure recovery under the assumption that the rows of $\Theta$ are bounded in $\ell_1$ norm by $M$, and \textsc{GreedyAndPrune} achieves a sample complexity of $O(\frac{M^4 \log(n/\delta)}{\kappa^6})$. We note that the former guarantee is as good as \cite{cai2016estimating}, which itself improves on the guarantee in \cite{cai2011constrained}.
\end{remark}
\section{Simulations and Experiments}\label{sec:simulations}
In this section, we will compare our proposed method (\textsc{GreedyAndPrune}) with popular methods previously introduced in the literature: the Graphical Lasso \cite{friedman2008sparse}, the Meinhausen-B\"uhlmann estimator (based on the LASSO) \cite{meinshausen2006high}, \textsc{CLIME} \cite{cai2011constrained}, and \textsc{ACLIME} \cite{cai2016estimating} (an adaptive version of \textsc{CLIME}). In the first subsection, we consider simple attractive GGMS and show that our method always performs well compared to previous methods and sometimes outperforms them considerably. In the second subsection, we compare the performance on a real dataset (from \cite{buhlmann2014high}) and show that our methods \textsc{HybridMB} and \textsc{GreedyAndPrune} again compare favorably. Our experiment also gives evidence that walk-summability is a reasonable assumption in practice.
\subsection{Simple attractive GGMs where previous methods perform poorly}
Three of the most popular methods for recovering a sparse precision matrix in practice are the Graphical Lasso (glasso) \cite{friedman2008sparse}, the Meinhausen-B\"ulhmann estimator (MB) based on the Lasso \cite{meinshausen2006high}, and the CLIME estimator \cite{cai2011constrained}. The graphical lasso is the $\ell_1$-penalized variant of the MLE (Maximum Likelihood Estimator) for the covariance matrix; CLIME minimizes the $\ell_1$-norm of the recovered precision matrix $\hat{\Theta}$, given an $\ell_{\infty}$ constraint $|\Sigma \Omega - Id|_{\infty} \le \lambda$ (where $|M|_{\infty} = \|M\|_{1 \to \infty}$ is the entrywise max-norm). For Meinhausen-B\"uhlmann, we let the estimated $\hat{\Theta}$ have its rows be given by the appropriate lasso estimate, scaled appropriately by the corresponding estimate for the conditional variance. The current theoretical guarantees of these methods have very high sample complexity for general GFFs and we find simple examples in which the scaling of their sample complexity with $n$ is poor. One example (which breaks the Meinhausen-B\"uhlmann estimator) is simply based off of a simple random walk observed at large times; the other examples we use are simple combinations of a path and cliques:
\begin{example}[Path and cliques]\label{example:path-clique}
Fix $d$ and suppose $n/2$ is a multiple of $d$. Let $B$ be a standard Brownian motion in 1 dimension, and
let $X_1,\ldots,X_{n/2}$ be the values of the $B$ at equally spaced points in the interval $[1/2,3/2]$, i.e. $X_1 = B(1/2),X_2 = B((1/2) + 1/(n - 1)), \ldots$ Equivalently, let the covariance matrix of this block be $\Cov(X_i,X_j) = 1/2 + \min(i,j)/n$, or take the Laplacian of the path and add the appropriate constant to the top-left entry.

Let the variables $X_{n/2 + 1,\ldots,X_n}$ be independent of the Brownian motion, and let their precision matrix be block-diagonal with $d \times d$ blocks of the form $\Theta_1$ where $\Theta_1$ is a rescaling of $\Theta_0$ so that the coordinates have unit variance, and $\Theta_0 = I - (\rho/d) \vec{1} \vec{1}^T$ where $\rho \in (0,1)$. In all experiments, we finally standardize the variables to have unit variance, following the usual recommendation (although the variances in this example are already bounded between $0.5$ and $1.5$).
\end{example}

The results of running all methods\footnote{For the Graphical Lasso we used the standard R packages recommended in the original papers. For \textsc{CLIME}, we originally tested the standard R package but it was unable to reconstruct a path, presumably due to numerical issues. To fix this, we reimplemented \textsc{CLIME} using Gurobi and used a similar implementation for \textsc{ACLIME}.} on samples from this model are shown in Figure~\ref{fig:path-clique} for the Frobenius error with a fixed number of samples ($m = 150$) where the clique degree is $d = 4$ and the edge strength is $\rho=0.95$. In Figure~\ref{fig:path-clique-samples} we show the number of samples needed to recover the true edge structure for the same example with $d = 4$ in two cases, $\rho=0.7$ and $\rho=0.95$. We note that our definition of structure recovery is fairly generous --- we apply a thresholding operation to the returned $\Theta$ matrix using the true value of $\kappa/2$, so the algorithms are not penalized for returning matrices with many small nonzero entries (which happens in practice at the optimal tuning of parameters, even though in the theory of e.g. \cite{meinshausen2006high} neighborhood estimates are made just from the support of the lasso estimate).

Note in particular that from Figure~\ref{fig:path-clique-samples}, we see the sample complexity of \textsc{GreedyPrune} scales like $O(\log(n))$, the information-theoretic optimal scaling which is in agreement with Theorem~\ref{thm:greedy-ferromagnetic}, while in the first example ($\rho = 0.7$) the sample complexity of the Graphical Lasso scales roughly like $\Theta(n)$ and in the second example ($\rho = 0.95$) the same is true for CLIME.


Recall that these examples are well-outside of the regime where the theoretical guarantees for methods like \textsc{CLIME} and Graphical Lasso can guarantee accurate reconstruction from $O(polylog(n))$ sammples, which is one reason we might expect them to be hard in practice. For example, the analysis of \textsc{CLIME} requires a bound on the entries of the inverse covariance (after rescaling the coordinates to have variance $\Theta(1)$), but for the path Laplacian the entries of the precision matrix are of order $\Theta(n)$. 

We describe one additional intuition as to why the Graphical Lasso should fails on this example: for the penalty $\lambda \|\hat{\Theta}\|_1$ to respect the structure of the path (where conditional variances are small) $\lambda$ should be chosen small, but then the nodes in the cliques may gain spurious edges to the path and other cliques. With \textsc{CLIME} there is a similar concern that the $\ell_1$ penalty for the two types of nodes does not scale properly. Different regularization parameters for the different types of edges could help in this particular example --- however, it is typically difficult know beforehand which nodes have small and big conditional variances without effectively learning the GGM, as the way to show a node has low conditional variance almost always involves finding a good predictor of it from the other nodes. Concretely, in the case of \textsc{ACLIME}, it performed significantly worse than  \textsc{CLIME} in most of our tests. On the other hand, the rescaling performed by our proposed algorithm \textsc{HybridMB} does resolve this issue in a principled way.  

In the above two examples we tried, the (thresholded) Meinhausen-B\"uhlmann estimator successfully achieved similar sample complexity to our proposed methods, despite the fact that this example is again well outside of the regime where its theoretical guarantees are good. However, as we see in Figure~\ref{fig:free-path} the sample complexity of this estimator is poor in another very simple example: a simple random walk with Gaussian steps run from times $n$ to $2n$. (As before, this is the description of the model before standardizing coordinates to variance $1$.) This is again not so surprising, as we know the Lasso (which the MB method is based upon) can only be guaranteed to obtain its ``slow rate'' guarantee when the coordinates of the input are highly dependent, and the slow rate guarantee for Lasso depends on norm parameters that are not sufficiently small in our example for good recovery guarantee. 
\begin{figure}
    \centering
    \hspace*{1.7cm}\includegraphics[scale=0.66]{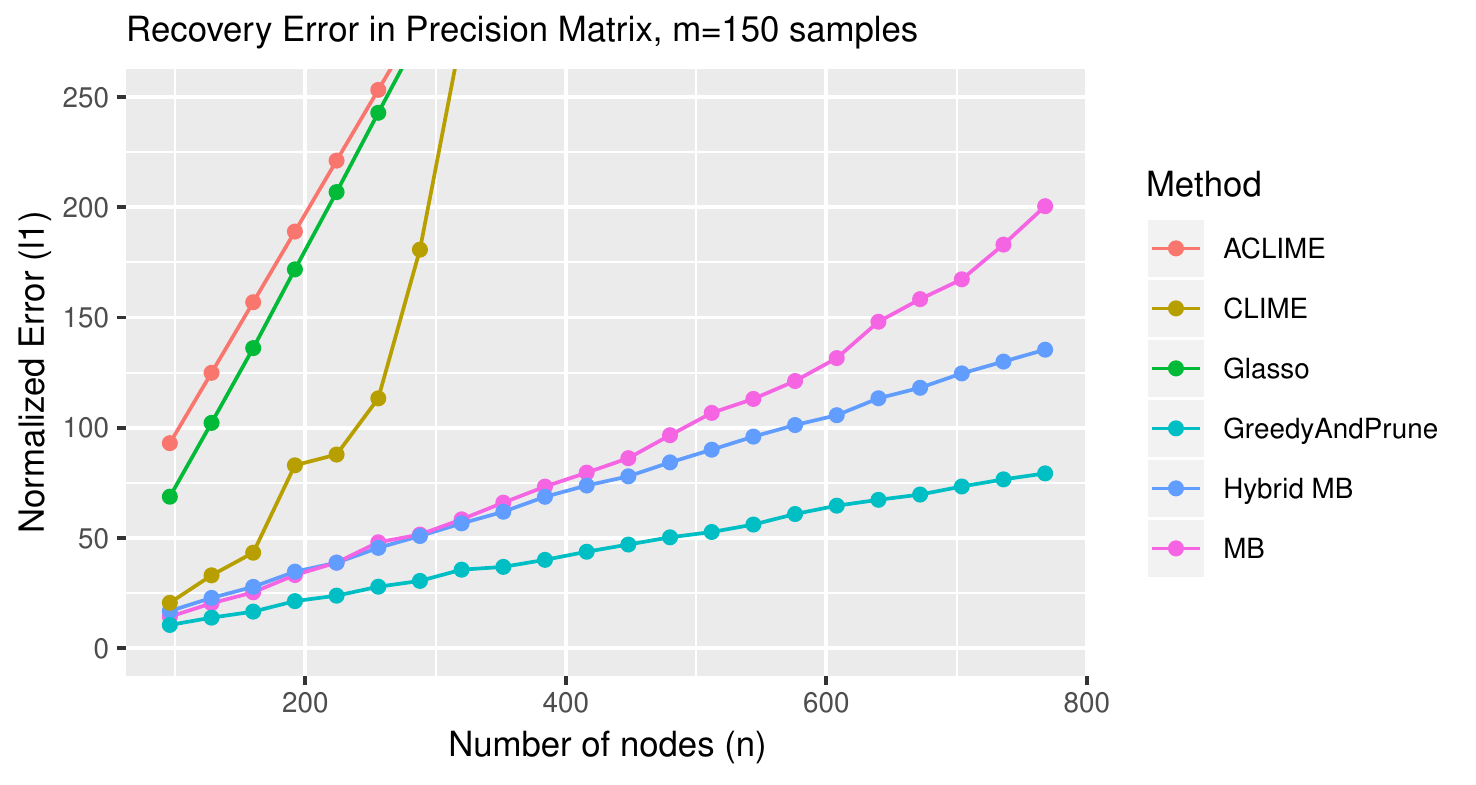}
    \caption{Normalized error (measured by $\|\hat{\Theta} - \Theta\|_1/n$ where $\|\cdot\|_1$ denotes the $\ell_1$ norm viewing the matrix as a vector) in the precision matrix returned in Example~\ref{example:path-clique} with $\rho = 0.95$. We note that this quantity should be expected to scale at least linearly, because some entries of $\Theta$ grow with $n$.
    Errors were averaged over 8 trials for each $n$
    and hyperparameters 
    were chosen by grid search minimizing the recovery error in a separate trial, for each value of $n$. The tested parameters for $\lambda$ in glasso were chosen from a log grid with 15 points from $0.0005$ to $0.4$, similarly for $\lambda$ in MB, from 8 points from $1$ to $32$ for $\gamma'$ in \textsc{HybridMB} (we set $\tau = 0$ for a more direct comparison to MB), for CLIME from a log grid with 15 points from $0.01$ to $0.8$, and for \textsc{GreedyAndPrune} $k$ from a rounded log grid with 7 points from $3$ to $24$ and $\nu$ from a log grid with 8 points from $0.001$ to $0.1$.}
    \label{fig:path-clique}
\end{figure}
\begin{figure}
    \centering
    \begin{subfigure}{.49\linewidth}
    \centering
    \hspace*{0.03\linewidth}
    \includegraphics[scale=0.57]{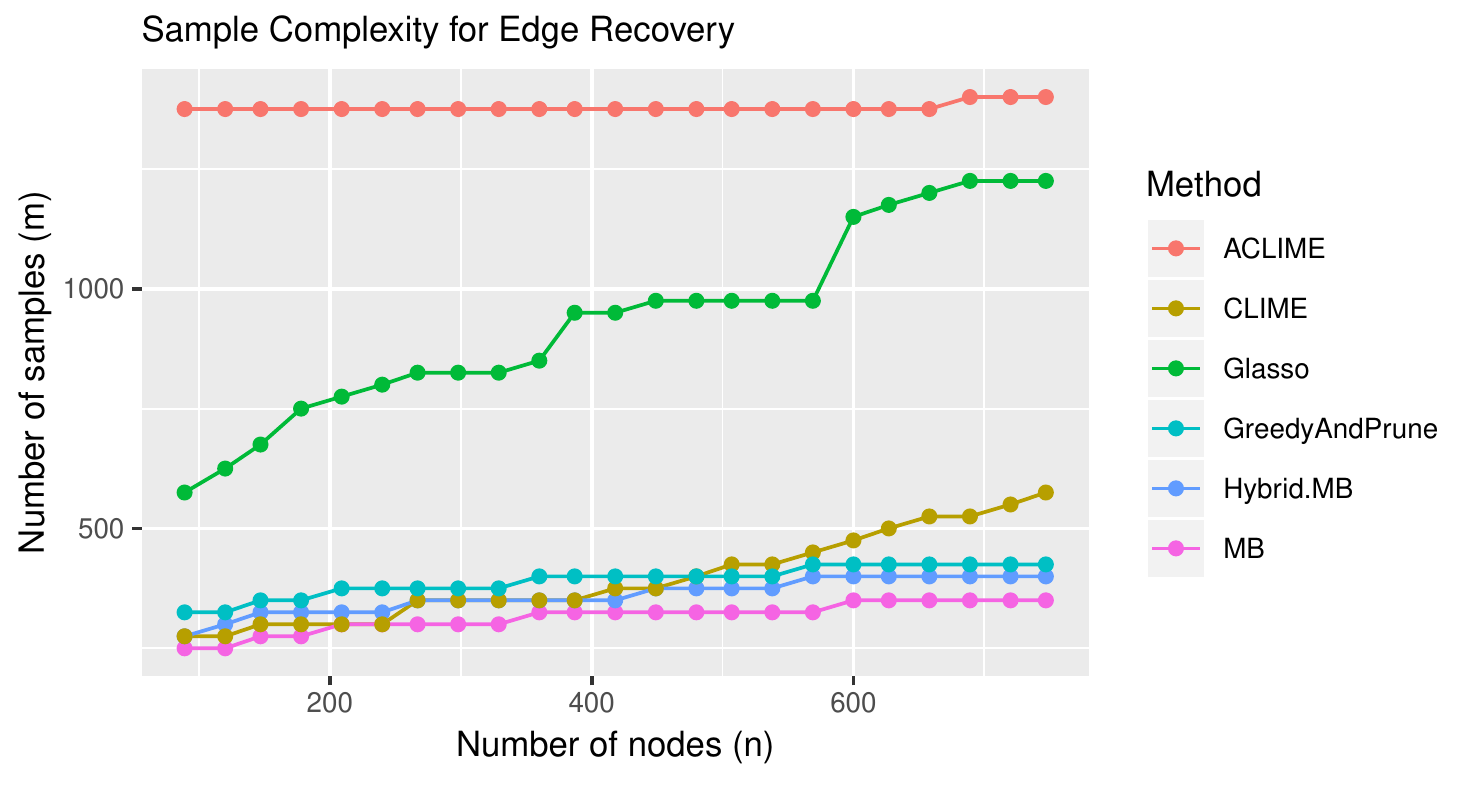}
    \caption{$d = 4$ and $\rho = 0.7$}
    \end{subfigure}
   \begin{subfigure}{.49\linewidth}
   \centering
   \hspace*{.05\linewidth}
   \includegraphics[scale=0.57]{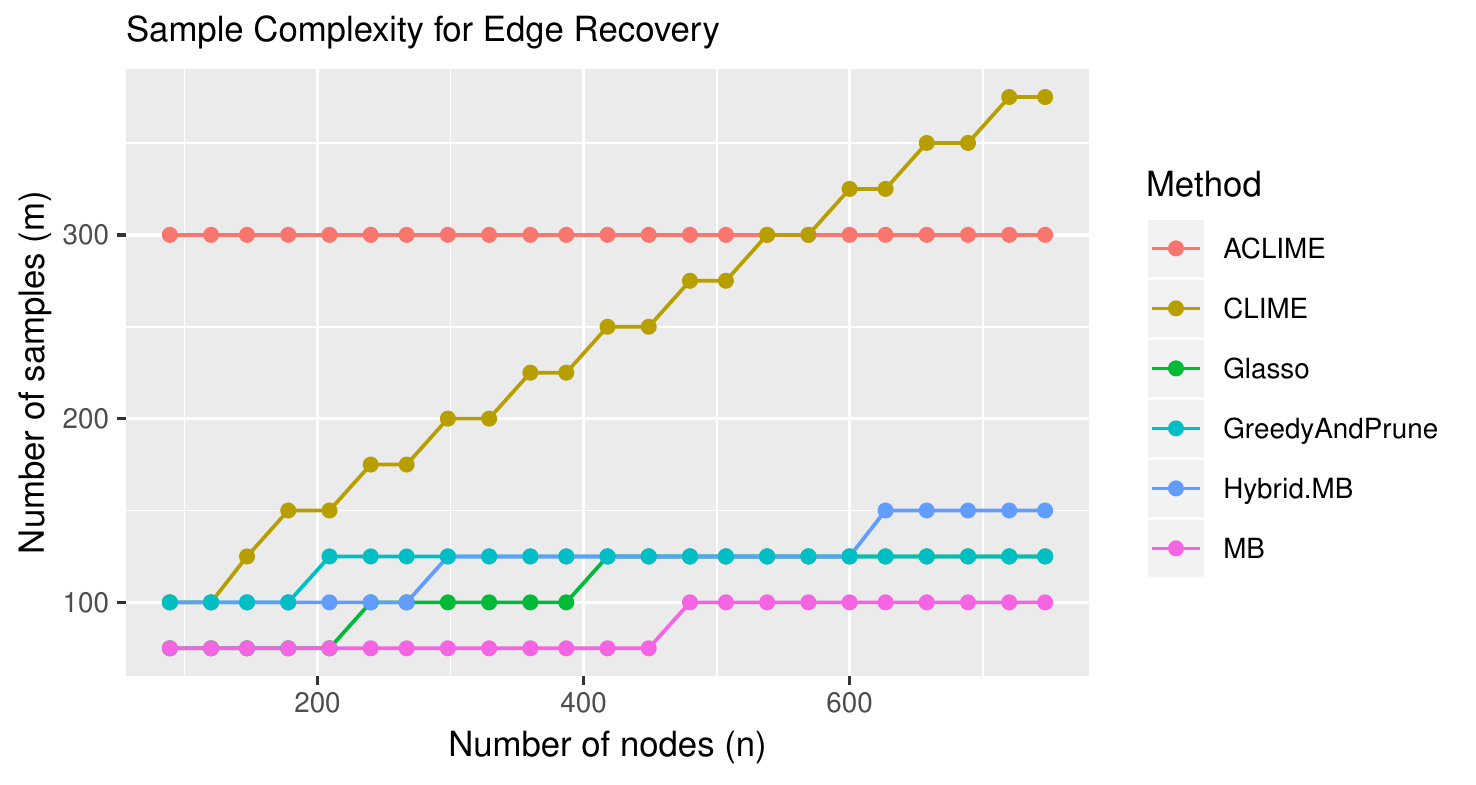}
   \caption{$d = 4$ and $\rho = 0.95$}
   \end{subfigure}
    \caption{Number of samples needed to approximately recover true edge structure after thresholding using the test $\frac{|\hat{\Theta}_{ij}|}{\sqrt{\hat{\Theta}_{ii}\hat{\Theta}_{jj}}} > \kappa/2$, where $\kappa$ is the $\kappa$ for the true precision matrix from the information-theoretic assumption \eqref{eqn:kappa-assumption}. Samples are drawn from the model in Example~\ref{example:path-clique} with two different values for the edge strength $\rho$.
    Note that the sample complexity of \textsc{GreedyPrune} is consistent with the $O(\log(n))$ bound established in Theorem~\ref{thm:greedy-ferromagnetic}, whereas the graphical lasso and CLIME have sample complexity that appears to be roughly $\Theta(n)$ in the left and right examples respectively.
    The $m$ shown is the minimal number of samples needed for the average number of incorrect edges per node (counting both insertions and deletions) to be at most 1. 
    Trials and parameter selection was performed the same way as in the experiment for Figure~\ref{fig:path-clique}, except that the parameters were chosen to minimize the number of incorrect edges, instead of error in the $\ell_1$ norm.}
    \label{fig:path-clique-samples}
\end{figure}
\begin{figure}
    \centering
    \includegraphics[scale=0.57]{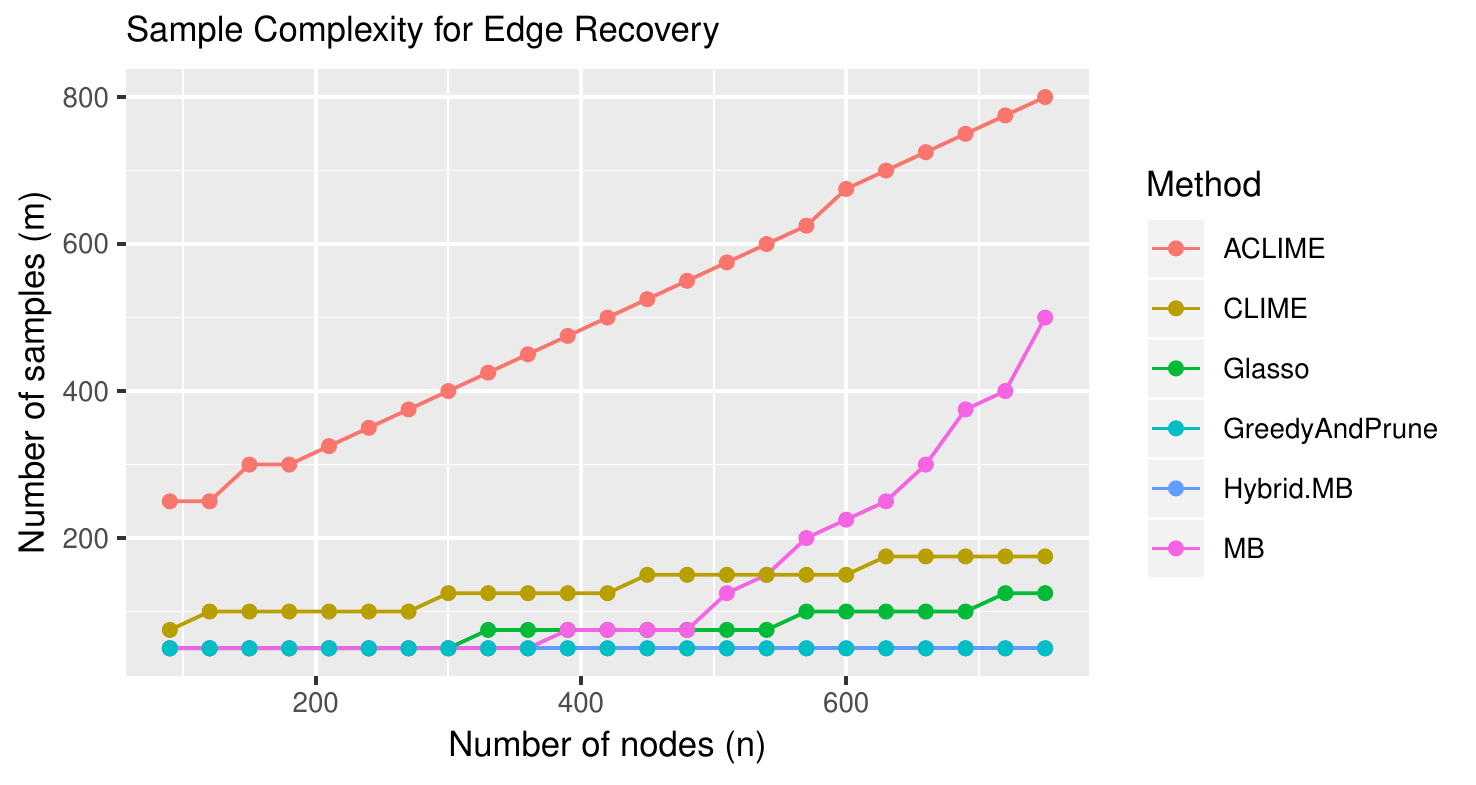}
    \caption{Large initial time simple random walk example: the setup is the same as in Figure~\ref{fig:path-clique-samples}, except that the ground truth model is a Gaussian simple random walk observed from times $n$ to $2n$. We observe in this example that the sample complexity of ACLIME and the Lasso-based Meinhausen-B\"uhlmann estimator appear to scale roughly linearly in $n$, whereas the sample complexity of \textsc{GreedyAndPrune} and \textsc{HybridMB} is in fact constant over the observed values of $n$. }
    \label{fig:free-path}
\end{figure}
\subsection{Results for Riboflavin dataset}
In this section we analyze the behavior of recovery algorithms
on a popular dataset provided in \cite{buhlmann2014high}. This
dataset has $m = 71$ samples and describes (log) expression levels for $n = 100$ genes in \emph{B. subtilis}. We compared all of the methods listed above; our tables do not list the \textsc{ACLIME} results because it did not achieve nontrivial reconstruction (it's CV error as defined below was $0.98$, which is essentially the same as the score for returning the identity matrix).
We selected parameters using a 5-fold crossvalidation with the following least-squares style crossvalidation objective\footnote{An alternative which is sometimes used is the likelihood objective $\Tr(\hat{\Sigma}\hat{\Theta}) - \log\det(\hat{\Theta})$, but this objective is not very smooth due to the $\log \det$ term and may equal $\infty$ even for entry-wise ``good'' reconstructions.}, after standardizing the coordinates to each have empirical variance $1$ and mean 0:
\[ E(\hat{\Theta}) := \frac{1}{nm_{holdout}} \sum_{i = 1}^n \sum_{k = 1}^{m_{holdout}} (X_i^{(k)} + \sum_{j \ne i} \frac{\hat{\Theta}_{ij} + \hat{\Theta}_{ji}}{2\hat{\Theta}_{ii}} X^{(k)}_i)^2. \]
Note that the true $\Theta$ minimizes this objective as $m_{holdout} \to \infty$, making it equal to the sum of conditional variances; when the initial variances are set to 1, this objective simply measures the average amount of variance reduction achieved over the coordinates. 

\begin{table}
    \centering
    \begin{tabular}{l|c|c|c|c|c|c}
         Method & CV Error & CV Parameters & \# Non-zeros & Cond. No. & $M$ & $\Delta_{WS}$ \\ 
         \hline
         Graphical Lasso & 0.13 & $\lambda = 0.01$ & 4378 & 968.6 &
         54.8 &  8.7 \%\\
         \textsc{CLIME} & 0.41 & $\lambda = 0.21$ & 806 & 193.8 & 232.2 & 0.0 \% \\
         \textsc{GreedyAndPrune} & 0.27 & $k = 13, \nu = 0.01$ & 476 & 389.4 & 224 & 1.1 \% \\
         \textsc{MB} & 0.17 & $\lambda = 0.05$ & 1854 & 21439 & 156 & 1.1 \% \\
         \textsc{HybridMB} & 0.19 & $\gamma' = 21$ & 2758 & 1080843 & 324 & 2.2 \% \\
    \end{tabular}

    \caption{Results for precision matrix selected via 5-fold CV on Riboflavin dataset. The last 4 columns give summary statistics for the final recovered $\hat{\Theta}$ using the CV parameters on the entire dataset: $M$ is the maximum $\ell_1$ row norm for any row of $\Theta$, the same as in the guarantee for CLIME cited earlier. The walk-summable relative error is $\Delta_{WS} := \frac{\|\tilde{\Theta} - \hat{\Theta}\|_F}{\|\hat{\Theta}\|_F}$ where $\tilde{\Theta}$ is the closest walk-summable matrix to $\hat{\Theta}$ in Frobenius norm. This shows that all of the estimated precision matrices are either walk-summable or close to walk-summable.}
    \label{table:riboflavin}
\end{table}
\begin{table}
    \centering
    \begin{tabular}{c|c}
         Method & Runtime (seconds) \\
         \hline
         Graphical Lasso & 0.74 \\
         \textsc{CLIME} & 2.12 \\
         \textsc{GreedyAndPrune} & 0.19 \\
         \textsc{MB} & 0.48 \\
         \textsc{HybridMB} & 1.84
    \end{tabular}
    \caption{Sequential runtime of methods on Riboflavin dataset with CV parameters, averaged over 10 runs. In all experiments, the graphical lasso implementation was from the glasso R package, CLIME was implemented by calling Gurobi from R (due to numerical limitations of the standard package), \textsc{MB} and \textsc{HybridMB} were implemented using the glmnet package, and for \textsc{GreedyAndPrune} we used a naive R implementation.}
    \label{tab:runtime}
\end{table}

The results of the cross-validation process\footnote{Essentially the same as before, parameters for Graphical Lasso were chosen from a log-scale grid from $0.001$ to $0.5$ with $15$ points, for CLIME similarly from $0.01$ to $0.8$ with $20$ points, and for \textsc{GreedyAndPrune} from a rounded log-scale grid from $3$ to $26$ with $7$ points and from $0.001$ to $0.1$ with $8$ points.} are shown in Table~\ref{table:riboflavin}. 
As we see from the first 2 columns of the table, Graphical Lasso achieved the greatest amount of variance reduction but returned the densest estimate for $\Theta$, \textsc{MB} and \textsc{HybridMB} had slightly less variance reduction, \textsc{GreedyAndPrune} had the sparsest estimate and achieved significantly more variance reduction that \textsc{CLIME}. We see that the chosen precision matrices have large condition number and row $\ell_1$-norm $M$, comparable to the number of nodes $n$, which is significant in that known guarantees for Graphical Lasso, MB, \textsc{CLIME} and \textsc{ACLIME} are only interesting when these quantities are small (e.g. constant or $O(\log n)$). (Equivalently, the gap between variance and conditional variance is large; we note that the true gap may be even larger if we had access to more data, since we might be able to find even better estimators for each $X_i$ given the other coordinates.) On the other hand, the recovered matrices are not far from walk-summable in Frobenius norm, suggesting that this is indeed a reasonable assumption.

In Table~\ref{tab:runtime} we record the sequential runtimes of all of the methods on this dataset using the CV parameters. \textsc{GreedyAndPrune} was the fastest method. For larger datasets it is important to use parallelism, and we note we note that CLIME, \textsc{MB}, \textsc{Hybrid.MB} and \textsc{GreedyAndPrune} are ``embarassingly parallelizable'', as each node can be solved independently, but this is not the case for the Graphical Lasso. In practice, on our synthetic datasets and using 24 cores, \textsc{CLIME} becomes faster than the Graphical Lasso and \textsc{GreedyAndPrune} stays the fastest. In our experiment, we did not test our proposed method \textsc{SearchAndValidate} or the methods of \cite{misra18}, although they have good sample complexity guarantees, due to computational limitations; in \cite{misra18}, they report their methods requires on the order of days to run on this example.

We also performed a ``semi-synthetic'' experiment on this dataset, by taking the recovered (dense) $\Theta$ from Graphical Lasso, thresholding it to have $\kappa = 0.15$ and computing the sample complexity to recover the edges of the graphical model from sampled data (as in the synthetic experiments, with error of at most $0.25$ incorrect edges per node, after thresholding at $\kappa/2$). All methods performed similarly on this test: the results are shown in Table~\ref{table:semisynthetic}. 
\begin{remark}
Several papers have been written on faster implementations of the graphical lasso, e.g. the Big \& Quic estimator of \cite{bignquic}. However, these methods have mostly been developed/tested in the regime where $\lambda$ is quite large: e.g. the documentation for the R package BigQuic implementing Big \& Quic suggests using $\lambda \ge 0.4$ and that $\lambda = 0.1$ is too small to run in a reasonable time on large datasets. In practice, these methods may even fail to return the true optimum when given small $\lambda$; however, the above experiment suggests this is an important regime in practice. 
\end{remark}
\begin{table}
    \centering
    \begin{tabular}{c|c|c}
         Method & Number of Samples Needed & Optimal Parameters \\
         \hline
         Graphical Lasso & 500 & $\lambda = 0.005$ \\
         \textsc{CLIME} & 550 & $\lambda = 0.04$ \\
         \textsc{GreedyAndPrune} & 550 & $k = 6, \nu = 0.01$ \\
         \textsc{MB} & 550 & $\lambda = 0.01$ \\
         \textsc{HybridMB} & 525 & $\gamma' = 21$
    \end{tabular}
    \caption{Number of samples needed to achieve error of at most $0.25$ incorrect edges per node after thresholding in the semi-synthetic experiment: samples were drawn from a $\Theta$ given by thresholding the graphical lasso estimate from the Riboflavin dataset. The details of the thresholding, etc. are the same as in the synthetic experiment of Figure~\ref{fig:path-clique-samples}. }
    \label{table:semisynthetic}
\end{table}

\begin{figure}
    \centering
    \begin{subfigure}{0.49\textwidth}
    \centering
    \includegraphics[scale=0.3]{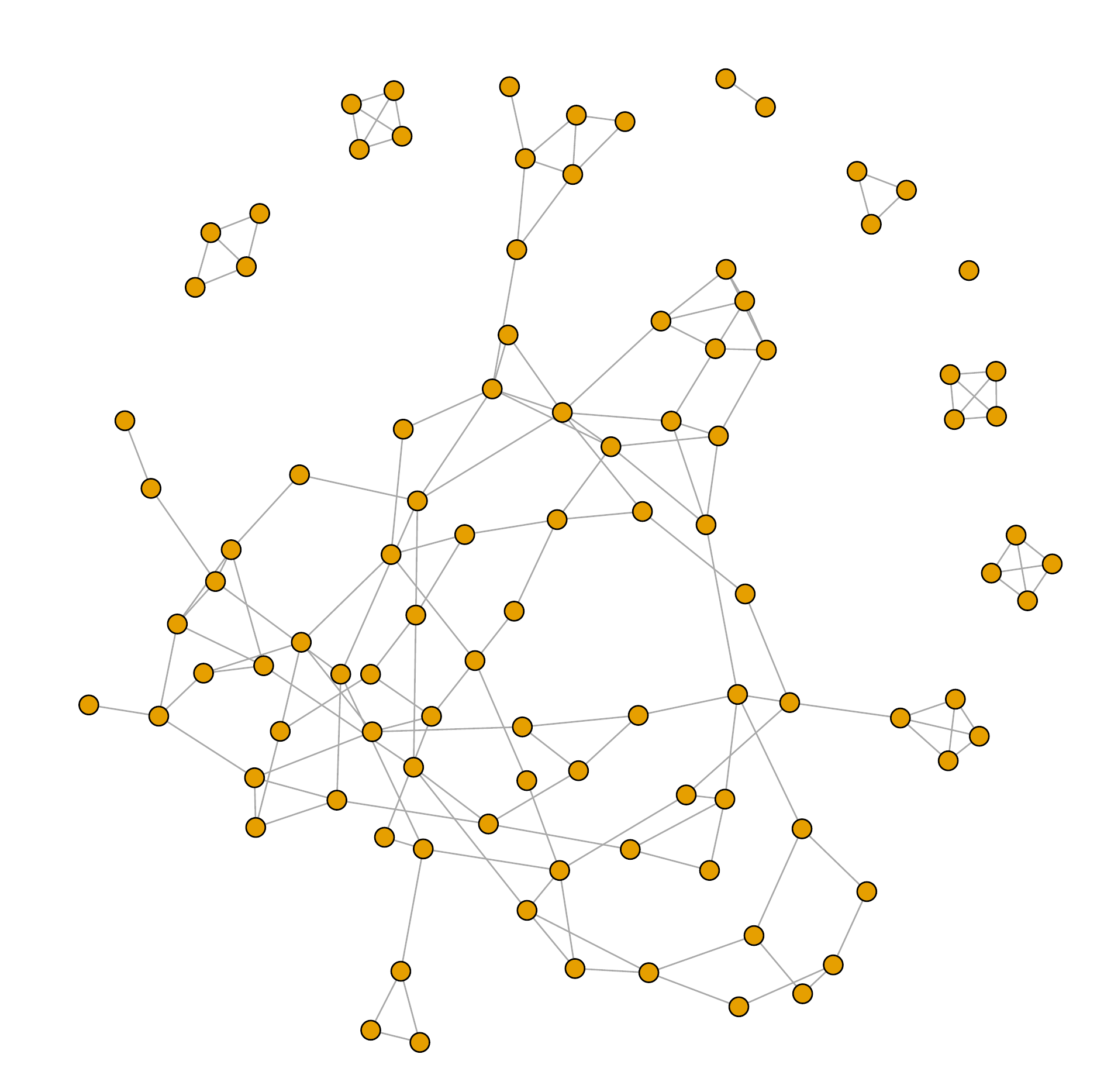}
    \end{subfigure}
    \begin{subfigure}{0.49\textwidth}
    \centering
    \includegraphics[scale=0.3]{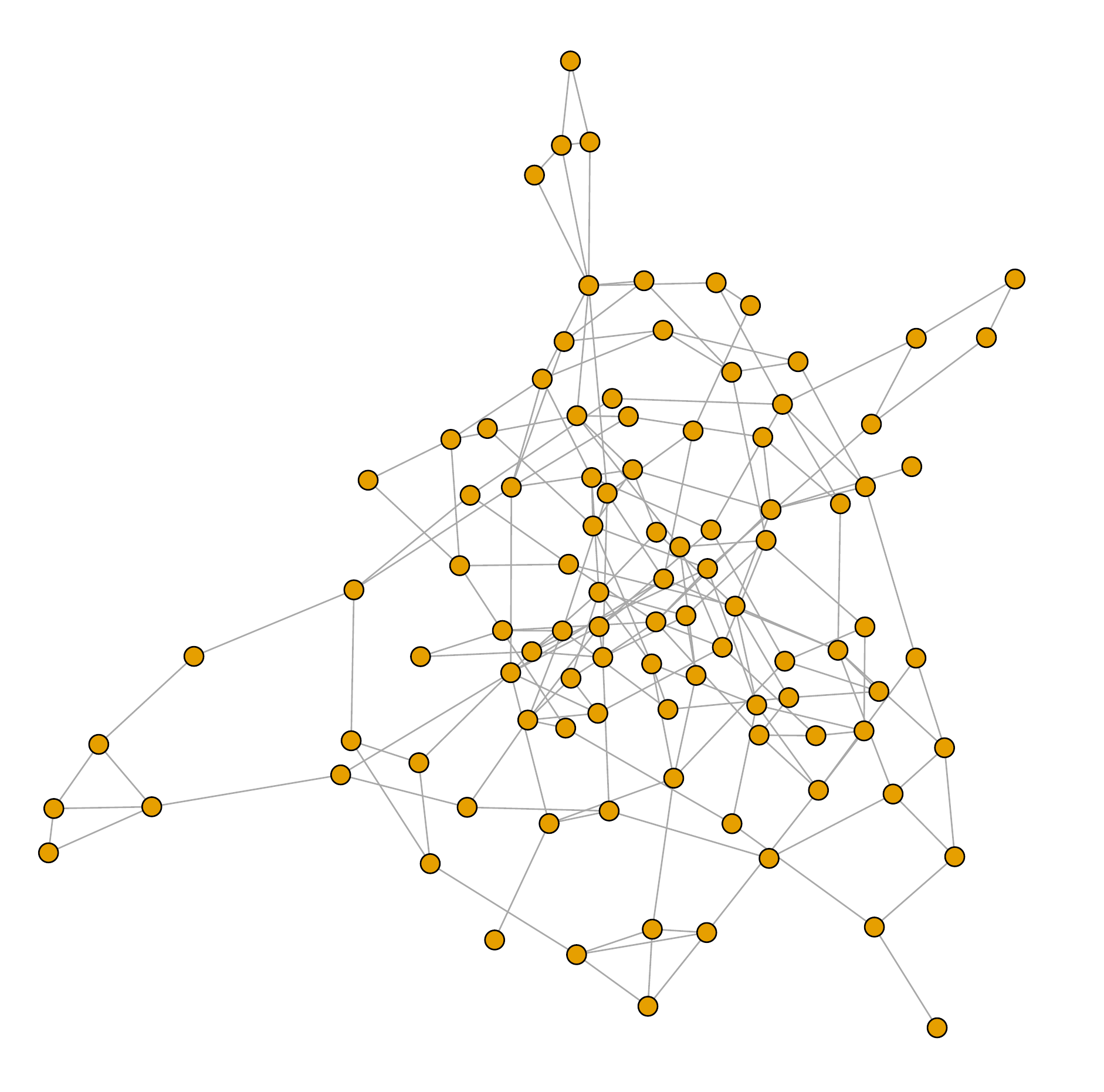}
    \end{subfigure}
    \caption{Left: thresholded graph from graphical lasso output on riboflavin data, used in semisynthetic experiment (see Table~\ref{table:semisynthetic}). Right: unthresholded graph output by \textsc{GreedyAndPrune} on Riboflavin data.}
    \label{fig:glasso-graph}
\end{figure}

\bibliographystyle{plain}
\bibliography{bib}
\appendix
\section{Some difficult examples}\label{apdx:hard-examples}
A natural question, given our previous results, is whether the simple greedy and prune algorithm could possibly learn all $\kappa$-nondegenerate GGMs with $O(\log n)$ sample complexity. Note by the analysis from Section~\ref{sec:greedy-l1} that if our greedy algorithm fails in an example, then any analysis based on bounded $\ell_1$-norm must also fail since greedy succeeds under a bounded $\ell_1$-norm assumption.

It is not too hard to find examples which break these algorithms when we view them as being run once from a single node, with the goal of recovering just that node's neighborhood. For example, if we take $n$ pairs of near-duplicate variables $(X_i,X'_i)$ with $\Var(X_i) = \Theta(n)$ and $\Var(X_i - X'_i) = \Theta(1)$ and define $Y = X_i - X'_i$ for some $i$, then using OMP to find a predictor of $Y$ will fail to find the edge from $X_i$ to $Y$ with $O(\log n)$ samples. However, if we run a greedy method to find a predictor of $X_i$, then we actually will discover this edge. In the following example, we see there are edges which are not discovered from either direction:
\begin{example}[Example breaking \textsc{GreedyAndPrune}]\label{example:break-greedy}
Fix $d > 2$ and let $Z_1,\ldots,Z_d$ be the result of taking $d$ i.i.d. Gaussians and conditioning on $\sum_i Z_i = 0$. Define $X_i = Z_i + \delta W_i$ and $Y_i = Z_i + \delta W'_i$ where $W_i,W'_i \sim N(0,1)$ independently. Let $\Sigma_0$ be the covariance matrix of $X_1,\ldots,X_d,Y_1,\ldots,Y_d$ (so the $Z$ are treated as latent variables).

It can be checked that the GGM with covariance matrix $\Sigma_0$ remains $\kappa$ nondegenerate for a fixed $\kappa$ even as $\delta$ is taken arbitrarily small.
Now consider the GGM which is block diagonal with first block $\Sigma_0$ and the second block the identity matrix, and suppose $n$ is large. If we try to learn the neighbors of $X_i$, greedy will with high probability fail to find a superset of the correct neighborhood of node $X_i$, because after conditioning on $Y_i$, the angles between the residual of $X_i$ and all of the other random variables are all near 90 degrees (going to 90 as $\delta \to 0$).
\end{example}
\begin{remark}
Part of the motivation for the use of nearly-duplicated random variables is that one can prove (using essentially a modified version of Lemma~\ref{lem:hybrid-greedy-step})) that in a general sparse GGM there always exists at least one node $i$ with at least one neighbor $j$ such that $\Var(X_i | X_j)$ is noticeably smaller than $\Var(X_i)$. In this example, this is trivially true but is not useful for discovering connections between unpaired variables.
\end{remark}
\begin{example}[Harder Example]\label{example:possibly-hard}
The previous example, while it breaks \textsc{GreedyAndPrune}, cannot be a hard example in general because the edge structure is easy to determine from the covariance matrix. (The covariance matrix is roughly block diagonal and each block corresponds to a clique). The following variant seems significantly harder: start with $\Sigma_0$ from the previous example, and then Schur complement (i.e. condition) out $d/4$ many of the nodes to yield $\Sigma'_0$. Then the covariance matrix of the whole model is block diagonal with $\Sigma'_0$ repeated $n/(d/4)$ times. Finally, we randomly permute the rows/columns.
\end{example}
Experimentally, it seems that Example~\ref{example:possibly-hard} breaks the methods considered in our experiments in the high-dimensional regime where the number of samples is much less than the dimension $n$. However, this example itself cannot be computationally hard to learn: a simple algorithm to learn it thresholds the covariance matrix to find the sub-blocks made up of the paired nodes from a block, then picks a sub-block, conditions it out, and finds the remaining nodes from this block as the nodes whose conditional variance went down significantly. 

\end{document}